\documentclass[twoside]{article}
\usepackage[accepted]{aistats2024}
\usepackage{microtype}
\usepackage{graphicx}
\usepackage{subfigure}
\usepackage{booktabs} 
\usepackage{float}

\usepackage[utf8]{inputenc} 
\usepackage[T1]{fontenc}    
\usepackage{hyperref}       
\usepackage{url}            
\usepackage{booktabs}       
\usepackage{amsfonts}       
\usepackage{nicefrac}       
\usepackage{microtype}      
\usepackage{amsmath}
\usepackage{xcolor}
\usepackage{graphicx}
\usepackage[english]{babel}
\usepackage{amsthm}
\usepackage{dsfont}
\usepackage{mwe}
\usepackage{wrapfig}
\usepackage{algorithm2e}
\usepackage{algorithmic}
\usepackage{tikz}

\newcommand\dkl{D_\text{KL}}

\usepackage{amsthm}
\theoremstyle{plain}
\newtheorem{theorem}{Theorem}[section]
\newtheorem{lemma}{Lemma}[section]

\newcommand\bx{\boldsymbol{x}}

\newcommand\bu{\boldsymbol{u}}

\usepackage{natbib}

\begin{document}
\twocolumn[

\aistatstitle{Privacy-Constrained Policies via Mutual Information Regularized Policy Gradients}

\aistatsauthor{ Chris Cundy* \And Rishi Desai* \And  Stefano Ermon }

\aistatsaddress{\url{cundy@cs.stanford.edu} \\  Stanford University \And  \url{rdesai2@cs.stanford.edu} \\ Stanford University \And \url{ermon@cs.stanford.edu} \\ Stanford University}

]

\begin{abstract}
  As reinforcement learning techniques are increasingly applied to real-world decision problems, attention has turned to how these algorithms use potentially sensitive information.
We consider the task of training a policy that maximizes reward while minimizing disclosure of certain sensitive state variables through the actions.
We give examples of how this setting covers real-world problems in privacy for sequential decision-making.
We solve this problem in the policy gradients framework by introducing a regularizer based on the mutual information (MI) between the sensitive state and the actions.
We develop a model-based stochastic gradient estimator for optimization of privacy-constrained policies. We also discuss an alternative MI regularizer that serves as an upper bound to our main MI regularizer and can be optimized in a model-free setting, and a powerful direct estimator that can be used in an environment with differentiable dynamics.
We contrast previous work in differentially-private RL to our mutual-information formulation of information disclosure.
Experimental results show that our training method results in policies that hide the sensitive state, even in challenging high-dimensional tasks.
\end{abstract}

\section{INTRODUCTION}

  Reinforcement learning (RL) algorithms have shown dramatic successes in areas such as game-playing~\cite{silverMasteringGameGo2016} and robotics~\cite{haarnojaSoftActorCriticOffPolicy2018}.
    This increasing adoption of RL techniques in real-world settings has illustrated
    the need for constraints on policies which are hard to express in the typical RL formulation, such as fairness~\cite{nabiLearningOptimalFair2019},
    risk-sensitivity~\cite{a.RiskSensitiveReinforcementLearning2018}, and safety under exploration~\cite{dalalSafeExplorationContinuous2018a}.    
    We aim to learn a policy to maximize reward, while minimizing the extent to which the policy's actions expose information about a sensitive subset of the state variables.

  This setting is natural to consider given the increasing prevalence of RL algorithms in society, many taking public actions based in part on sensitive internal information.
  Examples include a personal assistant algorithm setting appointments while not revealing important personal information, or a hospital administration algorithm assigning rooms to patients while avoiding disclosure of private medical details.

    A naive approach to this problem is to exclude the sensitive variables from the policy's input.
    However, it is well known~\cite{dworkCalibratingNoiseSensitivity2006} that this approach will fail when correlations exist between the sensitive and non-sensitive state variables.
    In the hospital example above, the policy may give room assignments based on non-sensitive variables such as age, which are correlated with medical status.
    Furthermore, the possibility of feedback in the setting introduces additional complications: the agent may be able to deliberately take actions to minimize future information disclosure.

    In the hospital scheduling example, this could involve investment so that all hospital rooms are equipped to treat any condition. Then the room assignment does not reveal any medical information.

  We formulate this privacy-constrained RL problem as an optimization problem with an additional regularizer on the mutual information between a function of the action \(a_t\) and a function of the protected state \(u_t\) at each timestep \(t\), induced under the learned policy \(q\).
    Optimizing this regularizer is not straightforward since it is distribution-dependent (unlike the reward), and involves marginalization over the non-sensitive state.
    We consider several different mutual information constraints and introduce gradient estimators, allowing privacy-dependent policies to be trained in the policy-gradients setting. First, we introduce a model-based gradient estimator for \(I_q(a_t;u_t)\), and a model-free gradient estimator for \(I_q(\tau_a, \tau_x;\tau_u)\), which serves as an upper-bound to \(I_q(a_t;u_t)\). Finally, we implement a reparameterized gradient estimator for \(I_q(a_{1:t};u_t)\) which can be applied given access to differentiable dynamics.
    Experiments show that our constrained optimization finds the optimal privacy-constrained policy in an illustrative tabular environment and hides sensitive state in a continuous control problem. Finally, we show that the reparameterized estimator can find policies which effectively hide the sensitive state in high-dimensional (simulated) robotics tasks.

\section{THE STATE-INFORMATION CONSTRAINED SETTING}
We analyse privacy-constrained sequential decision-making as a Markov decision process~\cite{suttonReinforcementLearningIntroduction2018} (MDP)
with information-theoretic constraints.
We consider episodic MDPs with a fixed length \(T\), with a state space \(\mathcal{S} = \mathcal{X} \times \mathcal{U}\) consisting of sensitive state variables \(u \in \mathcal{U}\) and non-sensitive variables \(x \in \mathcal{X}\).
In specific problems, the sensitive variables could include gender, location, or a sensitive health status.
In each episode, the initial state is drawn from a distribution \(p(x_1, u_1)\).
At each successive timestep the agent draws an action \(a_t\) from its policy \(q_\phi(a_t|x_t, u_t)\)
parameterized by learnable parameters \(\phi\) and the environment samples the successor state \((x_{t+1}, u_{t+1})\)
from the dynamics \(p\), giving a distribution over trajectories
\begin{align}
  q_\phi(\tau) = p(x_1, u_1)\prod_{t=1}^{T}q_\phi(a_t|x_t,u_t)p(x_{t+1}, u_{t+1}|a_t, x_t, u_t),
\end{align}
where \(\tau = (\tau_a, \tau_{x}, \tau_{u})\) is the collection of the actions and (sensitive and non-sensitive) states sampled in the trajectory.
At each timestep \(t \geq 1\), we obtain a reward \(r(x_t, u_t, a_t)\) concurrently with the transition to the next state.
In the standard formulation, the goal is to learn a policy \(q_\phi\) that results in a high expected reward.

As a concrete example, consider an RL algorithm acting as a virtual assistant, sending
emails and making appointments at each timestep.
The virtual assistant has available a set of variables describing the assistant's owner, of which some may be designated sensitive (e.g. private relationship status, trade secrets), and some may not be (e.g. current job, education status).
The assistant's goal is to take the best actions, corresponding to the most useful emails and appointments.
Furthermore, certain actions may change the state, affecting
future decisions. 
We are interested in learning a policy that maximizes cumulative reward while not allowing an
adversary to infer the values of the private state variables. In the simplest case, the adversary observes a single
action \(a_t\) and wishes to infer the sensitive variables at that time, \(u_t\) (not being
interested in the sensitive variables in the past or future), not observing any of the
state variables. We discuss alternative goals for the adversary in the next section.

We want to develop a worst-case guarantee, where the adversary cannot infer the value of \(u_t\) with full knowledge of the policy parameters \(\phi\) and the environment dynamics.
In this setting, a natural quantity to bound is the mutual information \(\mathbb{E}_{u_t, a_t \sim q_\phi}\left[\log q_\phi(u_t|a_t) - \log q_\phi(u_t)\right] = I_{q_\phi}(a_t;u_t)\), i.e. the amount of information an optimal adversary gains about \(u_t\) from observing \(a_t\)~\citep{liaoHypothesisTestingMaximal2017}.
By the data processing inequality, no adversary can gain more information about \(u_t\) from observing \(a_t\) than \(I_{q_\phi}(a_t;u_t)\), and so it serves as an upper bound on the information inferable by computation- or data-limited adversaries.
As a special case, if \(u_t\) is independent of \(a_t\), then the mutual information is zero. 

Therefore we formulate the problem of learning a state-private policy as a constrained optimization problem~\cite{altmanConstrainedMarkovDecision1999}, aiming to solve the problem
\begin{align}
    \label{eq:main-problem}  
  \hspace{-0.4cm} \underset{\phi}{\text{Max}} \ \underset{\tau \sim q_\phi}{\mathbb{E}}\left[\sum_{t=1}^Tr(x_t,u, a_t)\right], \quad \text{s.t.} \ I(a_t;u_t) < \epsilon_t \ \forall t,
\end{align}
with \(\epsilon_t\) a set of parameters 
that we can adjust to trade off good performance against low privacy.
Notably, \(I_{q_\phi}(a_t;u_t)\) does not involve the non-sensitive state \(x_t\).
In other words, we consider the mutual information between the actions and sensitive state,
marginalized over the distribution of non-sensitive state induced by \(q_\phi\). Furthermore the mutual information term is with respect to the distribution over trajectories induced by the policy \(q_\phi\), i.e.  \(I_{q_\phi}(a_t;u_t) = \dkl \left(
q_{\phi}(a_t, u_t)\|q_\phi(a_t)q_\phi(u_t) \right)\), with \(q_\phi(a_t,u_t) = \int_{\tau_{x_{1:t}}, \tau_{(a,u)_{1:t-1}}}
q_\phi(a_t|u_t, x_t)q_\phi(u_t,x_t|\tau_{(a, u)_{1:t-1}}, \tau_{x_{1:t-1}})\). This means that when
choosing the action at time \(t\), the agent must consider the effects on the
distribution at future timesteps, such as taking a corrective action that allows future
actions to be independent of \(u_t\).
\subsection{Alternative Threat Models}
\label{sec:alternative-threat-models}
In the previous section we discussed the threat model where the adversary aims to infer the
sensitive state at time \(t\), \(u_t\), by observing the corresponding action \(a_t\). This
corresponds to the case where the policy has a one-off interaction with an adversary where the
adversary does not aim to infer the previous or future states \(u_{t'}\) for \(t' \neq t\).
We can also consider adversaries that observe the current and previous actions \(a_{1:t}\) and wish to infer the current \(u_t\), or observe the whole trajectory of actions \(\tau_a\) and
wish to infer a single \(u_t\) or the whole sensitive trajectory \(\tau_u\).
Finally, an adversary might observe \(\tau_x\) and \(\tau_a\) and try to find \(\tau_u\). As incorporating more variables always increases mutual information,
\begin{align}
 & I_{q_\phi}(a_t;u_t) \leq I_{q_\phi}(a_{1:t};u_t) \leq I_{q_\phi}(\tau_a;u_t), \nonumber \\
& I_{q_\phi}(\tau_a;u_t) \leq I_{q_\phi}(\tau_a;\tau_u) \leq I_{q_\phi}(\tau_a, \tau_x; \tau_u).
\end{align}
Therefore, we can interpret \(I_{q_\phi}(\tau_a, \tau_x; \tau_u)\) as the relevant quantity to constrain if the adversary has access to all actions and non-sensitive states
and wishes to infer all sensitive states, or as an upper bound to any of the
MI quantities for other threat models. 
\subsection{Dual Formulation}
We can approach the constrained optimization problem~\eqref{eq:main-problem} by considering the Lagrangian
dual problem
\begin{align}
  \label{eq:main-dual-problem}
  \underset{\boldsymbol{\lambda} \geq 0}{\min}\ \underset{\phi}{\max} \ & \mathbb{E}_{\tau \sim q_\phi}   \left[R(\tau)\right] - \sum_{t=1}^T\lambda_t(I_{q_\phi}(a_t;u_t) - \epsilon_t),
\end{align}
where \(\boldsymbol{\lambda}\) is a vector of Lagrange multipliers. 
For a constrained optimization problem, the solution to the dual problem is a lower bound on the primal problem, which is not necessarily tight. However, for an important set of privacy-preserving problems, the bound is tight:
\begin{theorem}
  \label{sec:dual-formulation}
  For a time-dependent policy, \(q_\phi^t(a_t|x_t,u_t)\) in an MDP where \(u_t\) is independent of actions, equations \eqref{eq:main-dual-problem} and \eqref{eq:main-problem} have the same solution, i.e. strong duality holds between the primal and dual.
\end{theorem}
\begin{proof}
  Full details are in the appendix. The proof follows \citet{paternain2019constrained} which requires Slater's conditions and a concave perturbation function. The perturbation function \(P(\boldsymbol{\xi})\) is the value attained in equation \eqref{eq:main-dual-problem} with constraints \(-\boldsymbol{\xi} + \boldsymbol{\epsilon}\). Concavity follows from the convexity of the mutual information, and Slater's condition is satisfied by a random policy.
\end{proof}
Although we cannot prove strong duality in the general case, it holds empirically in our experiments for all cases where we can compute the optimal policy analytically.
In order to solve the inner maximization problem with gradient descent, we require estimators
for \(\nabla_\phi \mathbb{E}_{\tau \sim {q_\phi}}\left[R(\tau)\right]\) and \(\nabla_\phi I_{q_\phi}\), where \(I_{q_\phi}\) can be any of the mutual information quantities mentioned above.
Since even evaluating the mutual information in high dimensions is challenging~\cite{paninskiEstimationEntropyMutual2003}, this
is not trivial. We provide three different approaches:
a method for estimating \(\nabla_\phi I_{q_\phi}(a_t;u_t)\) where a dynamics model is available, a model-free gradient estimator for \(I_{q_\phi}(\tau_a, \tau_x;\tau_u)\), and a reparameterization based gradient estimator for any mutual information given a differentiable simulator. 

\section{RELATED WORK}
\label{sec:demographic-parity}
\subsection{Privacy In Reinforcement Learning}
The specific problem of satisfying privacy concerns while maximizing reward in a reinforcement learning context was introduced in~\cite{sakumaPrivacypreservingReinforcementLearning2008} and \cite{zhangPrivacyPreservingLearning2005}. Since then, several works have tackled the privacy-preserving RL problem in various special cases, such as linear contextual bandits \cite{neel2018mitigating,shariffDifferentiallyPrivateContextual2018}, multi-armed bandits \cite{sajedOptimalPrivateStochasticMAB2019,tossouAlgorithmsDifferentiallyPrivate2016}, and online learning with bandit feedback \cite{agarwalPriceDifferentialPrivacy2017,smithNearlyOptimalAlgorithms2013}. In non-bandit settings (i.e. the `general RL' setting) there is less work, most recently \citet{wangPrivacypreservingQLearningFunctional2019} and \citet{vietri2020private}, discussed below.

These works all use the differential privacy (DP) privacy metric. As summarized in \cite{basu2019differential}, a bandit algorithm is (globally) \(\epsilon\)-DP if \(\log q_\phi(\tau_a|\tau_x) - \log q_\phi(\tau_a|\tau_{x}') \leq \epsilon\) for all \(\tau_a\) and \(\tau_x, \tau_x'\) where \(\tau_{x}'\) is a trajectory that differs from \(\tau_x\) at only one timestep.
In general the relationship between DP and mutual information privacy constraints is not straightforward, although characterisations have been made in several settings \cite{wangRelationIdentifiabilityDifferential2016,mirInformationTheoreticFoundationsDifferential2012,dupincalmonPrivacyStatisticalInference2012}.
To our knowledge, a standard definition of DP privacy in the general RL setting is not agreed upon. Comparing to the DP constraint in the bandit setting, a key difference is that our constraint penalizes predictability of \(u_t\) given \(a_t\), in expectation over the distribution of \(u_t, a_t\), while a DP constraint penalizes predictability between neighbouring trajectories, with no notion of the relative likelihood of these trajectories. This is a particularly important difference in the general RL setting, where the ability of a policy to change the distribution of states is a key feature.

In \citet{wangPrivacypreservingQLearningFunctional2019}, the DP constraint is applied on the Q-learning algorithm itself, viewed as a function \(\mathcal{A}: \mathcal{R} \to \mathcal{Q}\) mapping a reward function to a Q-function. The DP requirement is that for any reward functions \(r, r'\) with \(\|r - r'\|_\infty < 1\), \(\log p(\mathcal{A}(r)) - \log p(\mathcal{A}(r')) < \epsilon\), with an RHKS measure over \(\mathcal{Q}\). We compare policies learned under this constraint to policies satisfying our MI constraint in section \ref{sec:experiments}. In the offline RL setting of \citet{qiao2024offline}, the DP constraint is applied between the individual trajectories in the training dataset and the resulting policy.

Finally, a DP constraint for general RL is described in ~\cite{vietri2020private}.
There, \(T\) episodes of length \(H\) are experienced, each with arbitrary dynamics and rewards. The constraint is that \(\log q_{\phi}(\tau_{a \setminus t}) - \log {\tilde q}_{\phi}(\tau_{a \setminus t}) \leq \epsilon\), for all \(\tau_{a \setminus t}\) (denoting the set of \(H(T-1)\) actions not including the actions in episode \(t\)) and all \(q, {\tilde q}\) (an MDP \(q\) and an MDP \({\tilde q}\), differing from \(q\) only in the \(t\)th episode). This is significantly more adversarial than our approach, which assumes a fixed MDP.

\subsection{Mutual Information Constraints in RL}
MI constraints have been used for reinforcement learning in the context of goal-directed RL. In the goal-directed setting, the agent has a goal \(g_t\) which affects the agent's choice of action, but not the dynamics. Previous work regularizes the MI  between goals and other quantities.
First, \cite{vandijkGroundingSubgoalsInformation2011} explored this in the options framework, regularizing \(I(a_t;g_t|\tau_{a_{1:t-1}}, \tau_{x_{1:t-1}})\).

More recently, both \cite{goyalInfoBotTransferExploration2019} and \cite{strouseLearningShareHide2018} studied the behaviour of policies regularized with the term \(I(a_t;g_t|s_t)\).
Both explore a rearrangement of this regularizer as the KL-divergence between the learned policy and a `default' policy.
In the case of \cite{goyalInfoBotTransferExploration2019} this is used to learn policies with diverse goals, similarly to the information bottleneck learning framework.
In the case of \cite{strouseLearningShareHide2018} the motivation is explicitly to learn agents that either share or hide action-goal information (depending on the sign of the regularizer's coefficient).
Although a similar motivation, the threat model considered is different: in~\citet{strouseLearningShareHide2018}, the adversary knows the state at time \(t\) and aims to infer the goal from the actions, while in our approach the state is unobserved and the adversary wants to infer a subset of the state from the actions. Perhaps closest to our work is the case discussed where actions are
unobserved--controlling \(I(s_t;g_t)\). Our approach solves the corresponding problem for unobserved states--controlling \(I(a_t;g_t)\), with the additional aspect that our `goal' may influence the dynamics of the environment. 

Finally, \citet{grau2018soft} apply a penalty of \(I(a_t;s_t)\) to the reward to encourage adaptive exploration. We compare to this method in section \ref{sec:experiments}. In contrast, our penalty is with respect to a private subset of the state variables. This increased selectivity means we can achieve higher reward by allowing disclosure of non-private state variables.

\subsection{Demographic Parity}
In fair machine learning, the demographic parity objective~\citep{zemelLearningFairRepresentations2013} for binary prediction requires that the class predicted, \(\hat y\), is statistically independent of protected variables such as race or gender.
Previous work \citep{songLearningControllableFair2019} has formulated this as requiring \(I(\hat y ; u) = 0\) for protected variables \(u\), marginalizing over the unprotected variables \(x\). Our approach is equivalent to the demographic parity objective for one-timestep episodes (see section \ref{sec:equiv-demogr-parity}).
It is unclear whether extending demographic parity from the bandit setting makes sense as a notion of fairness: different formulations of fairness for sequential decision-making have been proposed, such as meritocratic fairness \citep{jabbariFairnessReinforcementLearning2017} or path-specific fairness~\citep{nabiLearningOptimalFair2019}.

\section{OPTIMIZATION OF PRIVACY CONSTRAINTS}
It is well known that in general it is intractable to compute the MI between two random variables~\cite{paninskiEstimationEntropyMutual2003}.
The reinforcement learning setting provides us with an additional challenge, as to perform efficient gradient-based 
optimization we must form explicit Monte-Carlo estimators of the gradient in terms of distributions
that we are able to sample from. Two common tricks in working with mutual information are
to approximate posterior distributions with adversarial training~\cite{nowozinFganTrainingGenerative2016}, and to form upper bounds by introducing auxiliary
variables. We use each trick to obtain two different gradient estimators for our objective. Finally, we introduce an estimator for the case where a differentiable simulator is available.
\subsection{Estimation of the MI constraint}
\label{sec:per-timestep}
Our simplest mutual information constraint is \(I_{q_\phi}(a_t; u_t) \leq \epsilon_t\), for all \(1 \leq t
\leq T\).
By definition, \(I_{q_\phi}(a_t; u_t) = \mathbb{E}_{a_t, u_t \sim q_\phi}\left[\log
  q_{\phi}(u_t|a_t) - \log q_\phi(u_t)\right]\).
In general, there is no way to obtain these
probabilities in closed-form in terms of \(\phi\), since we only know \(q_\phi(a_t|x_t,
u_t)\).
However, we can replace \(q_\phi(u_t|a_t)\) and \(q_\phi(u_t)\) with approximating distributions
\(p_\psi(u_t|a_t)\) and \(p_\psi(u_t)\).
We learn the parameters \(\psi\) of \(p_\psi\) by maximum likelihood on samples from \(q_\phi(\tau)\).
Given a sufficiently powerful model \(p_\psi\) and enough samples, \(p_\psi(u_t|a_t) \approx q_\phi(u_t|a_t)\)  and \(p_\psi(u_t) \approx q_\phi(u_t)\), so \(I_{q_\phi}(a_t; u_t) \approx \mathbb{E}_{a_t, u_t \sim q_\phi}\left[\log p_\psi(u_t|a_t) - \log p_\psi(u_t)\right]\).
If \(p_\psi\) recovers \(q_\phi\) exactly we achieve the equality.
By training a predictor \(p_\psi\) we can check if a policy
is fulfilling the MI constraint in equation~\eqref{eq:main-problem}. For an alternative MI such as \(I_{q_\phi}(a_{1:t};u_t)\), we can similarly train a predictive model for \(u_t\) given \(a_{1:t}\). 
\subsection{Model Based Estimation of the MI Constraint Gradient}
In order to perform gradient-based constrained optimization to solve the problem in equation~\eqref{eq:main-problem}
we need a tractable gradient estimator.
If our policy is parameterized with \(\phi\), applying the policy
gradient theorem to \(\mathbb{E}_{a_t, u_t \sim q_\phi}\left[\log p_\psi(u_t|a_t) - \log
p_\psi(u_t)\right]\) gives us a gradient estimator \(\mathbb{E}_{a_t, u_t \sim q_\phi}\left[\left(\log p_\psi(u_t|a_t) - \log p_\psi(u_t)\right) \nabla_\phi \log q_\phi(a_t, u_t)\right]\).
However, \(\nabla_\phi \log q_\phi(a_t, u_t)\) is difficult to compute, as \(q_\phi(a_t, u_t)\) involves a marginalization over all previous states and actions in the trajectory. 
In Section \ref{sec:model-based-gradient} we show that
\begin{align}
  \label{eq:model-gradient}
  & \nabla_\phi \underset{{a_t, u_t \sim q_\phi}}{\mathbb{E}}\left[\log p_\psi(u_t|a_t) - \log p_\psi(u_t)\right] \\ = &  
 \nonumber \underset{a_t, u_t \sim q_\phi}{\mathbb{E}}\left[R_\psi(u_t, a_t)\underset{x_t \sim q_\phi(\cdot|u_t, a_t)}{\mathbb{E}}\left[
    \nabla_\phi\log q_\phi(a_t|x_t, u_t)\right.\right. \\
  &  \hspace{-0.5cm} + \underset{\tau_{(x,u,a)_{1:t-1}}\sim q_\phi(\cdot, \cdot|x_t, u_t)}{\mathbb{E}}\left[ \nonumber
    \sum_{t'=1}^{t'=t-1}\frac{p(x_{t'+1}, u_{t'+1}|a_{t'}, x_{t'}, u_{t'})}{q_\phi(x_{t'+1}, u_{t'+1}|x_{t'}, u_{t'})}\right.\\
  & \hspace{4.2cm}  \times \nabla_\phi \log q_\phi(a_{t'}|x_{t'}, u_{t'})\Bigg]\Bigg]\Bigg], \nonumber
\end{align}
where \(R_\psi(u_t, a_t) = \log p_\psi(u_t|a_t) - \log p_\psi(u_t)\).
To compute this estimate we need to know the transition dynamics of the MDP, \(p(x_t, u_t|a_{t-1}, x_{t-1}, u_{t-1})\).
Although the model-based requirement may seem stringent, it is plausible that a model will be available in higher-stakes settings where privacy is a consideration.
Furthermore, model-based techniques are increasingly popular due to empirical~\cite{kaiserModelBasedReinforcementLearning2019} and theoretical~\cite{duGoodRepresentationSufficient2019} sample-efficiency improvements over model-free techniques.
Since we can compute \(q_\phi(x_t, u_t|x_{t-1}, u_{t-1})\) as \(\int_{a_{t-1}}p(x_t, u_t|a_{t-1}, x_{t-1}, u_{t-1})q_\phi(a_{t-1}|x_{t-1}, u_{t-1})da_{t-1}\), a separate \(q_\phi(x_t, u_t|x_{t-1}, u_{t-1})\) is not needed. A similar estimator can be constructed for \(I_{q_\phi}(\tau_a; \tau_u)\) (details in the appendix, section \ref{sec:itau_a-tau_u-regul}).
\subsection{Action-Trajectory Mutual Information Constraint Gradient}
\label{sec:action-trajectory-mutual-information}
We can avoid the model-based marginalization in the previous section by
explicitly including the trajectory of actions and states in our mutual information term, and
so considering the constraint \(I_{q_\phi}(\tau_x, \tau_a; \tau_u) \leq \epsilon\).
As discussed in section~\ref{sec:alternative-threat-models}, this constraint is an upper bound for the constraint in equation~\eqref{eq:main-problem}, as well as an interesting constraint itself.

Similarly to above, we approximate \(q_\phi\) with a learned predictor \(p_\psi\), so \(I_{q_\phi}(\tau_a, \tau_x;\tau_u) \approx \mathbb{E}_{\tau\sim q_\phi} \left[\log p_\psi(\tau_u|\tau_{x}, \tau_a) - \log p_\psi(\tau_u)\right]\), for a sufficiently accurate predictor \(p_\psi\). This has a tractable gradient:
\begin{align}
\nabla_\phi \underset{\tau\sim q_\phi}{\mathbb{E}} \left[R_\psi(\tau)\right] = \underset{\tau\sim q_\phi}{\mathbb{E}} \left[R_\psi(\tau)\nabla_\phi \log q_\phi(\tau) \right],
  \end{align}
  where \(R_\psi(\tau) = \log \frac{p_\psi(\tau_u|\tau_{x}, \tau_a)}{p_\psi(\tau_u)}\). We can use this estimator to learn policies solving the constrained optimization problem in equation~\eqref{eq:main-problem}. We can optimize this quantity without any knowledge of the dynamics. If we aim to use this as an upper bound, the tradeoff is that the upper bound may be loose. In the appendix (section \ref{sec:u-shielded}) we examine the looseness of the bound in the setting where \(u_t\) influences the transitions and rewards only through the initial state.

\subsection{Estimation of MI Constraint Gradient with Differentiable Simulator}
With differentiable simulation environments \citep{freeman2021brax, hu2019difftaichi} it is possible to form path-based (reparameterization) gradient estimators. Provided the policy and dynamics can be written as differentiable functions of random variables \(\zeta\) (written \(\tau(\zeta, \phi)\)), we have, for any differentiable functions \(f, g\), 
\begin{align}
  & \nabla_\phi \underset{{a_{1:t}, u_t \sim q_\phi}}{\mathbb{E}}\left[I_{p_\psi}(f(\tau_a, \tau_x);g(\tau_u)) \right]\nonumber \\
& = \underset{{\zeta\sim p(\zeta)}}{\mathbb{E}}\left[\nabla_\phi I_{p_\psi}(f(\tau_a(\zeta, \phi), \tau_x(\zeta, \phi));g(\tau_u(\zeta, \phi))) \right], \nonumber
\end{align}
where \(\nabla_\phi I_{p_\psi}(f(\tau_a(\zeta, \phi), \tau_x(\zeta, \phi));g(\tau_u(\zeta, \phi)))\) can be computed directly with automatic differentiation.
In particular, we investigate the statistic \(I(a_{1:t};u_t)\), i.e. how well an adversary can guess the hidden state at time \(t\) after seeing all the actions up to that time. The main concern with this method is the variance of the gradient estimator if the dynamics of the simulator are stiff. In the case of rigid body simulations, the gradient estimator can even have much higher variance than the corresponding score function estimator \cite{suh2022differentiable}, and the variance can increase rapidly with \(t\).
We explore this in our experiments, and show that a simple truncation in backpropogation can control the variance sufficiently to produce complex state-hiding policies in high-dimensional simulated robotics tasks.

\section{EXPERIMENTS}
\label{sec:experiments}
In this section we use our constrained optimization procedure to solve privacy-constrained tasks in several different environments. Using the model-based score function estimator, we first consider a tabular task to illustrate that we can learn policies that intelligently plan ahead, changing the distribution over future states in order to reduce the information leaked in subsequent timesteps.
We compare the behaviour of a differentially-private Q-learning policy to our mutual information-constrained policy. Additionally, we compare to a previous MI-constrained method, MIRL \citep{grau2018soft}. 

Going beyond tabular environments, we evaluate the model-based method on a two-dimensional control task.
Finally, we deploy the reparameterized method on two simulated robotics tasks with a differentiable simulator and PPO, a state-of-the-art RL policy-gradient approach, and show we can learn high-reward policies which effectively hide the hidden state. 
In the appendix, additional experiments compare the behaviour of the \(I(\tau_u; \tau_x, \tau_a;)\) constraint to the \(I(u_t;a_t)\) constraint on a toy example and investigate a welfare-allocation task.

\subsection{Privacy in Internet Connections}
\label{sec:priv-intern-conn}
As a first goal,
we want to illustrate that the learned policies exploit the structure of the problem in order to achieve the privacy constraints, such as taking early corrective actions which remove future \(u\)-dependence from actions.
To show this, we construct a tabular example representing connection to various web sites.
The agent has one of \(n\) IP addresses, which are considered private. At each of \(T\) timesteps, the agent has a choice of connecting to the websites via \(n\) \emph{mirrors}.
The mirror corresponding to the current IP address will connect quickest, resulting in highest reward \(r^*\).
Connecting via the other mirrors results in a slower or intermittent connection, with lower reward \(r^- \ll r^*\).
The agent can also purchase a VPN, which gives no immediate reward but allows reasonably good connection to all mirrors at future timesteps, with reward \(r^- \ll r^{\text{VPN}} < r^*\) for connecting to any mirror.
The binary non-sensitive state \(x \in \{0, 1\}\) represents whether the VPN has been purchased or not.

The unconstrained optimal policy is simply to always choose the mirror corresponding to the owner's IP address, resulting in a total reward of \(Tr^*\).
The optimal policy\footnote{Providing that \(\tfrac{T-1}{T} > \tfrac{r^-}{r^{\text{VPN}}}\), which is the case for our setup with \(T=10, r^{\text{VPN}}=0.9, r^-=0.5\)} under a strict privacy constraint on the IP address is to choose to activate the VPN on the first timestep, then choose any of the mirrors under the subsequent timesteps, resulting in a total reward of \((T-1)r^{\text{VPN}}\).
In our experiments we used \(n=4\).
We solve problem~\eqref{eq:main-dual-problem} with \(\epsilon_t = \infty\) (non-privacy constrained) and \(\epsilon_t = 0\) (privacy-constrained), using the model-based score-function estimator in equation \eqref{eq:model-gradient}.
In this simple setting, we use empirical frequencies of \(u_t, a_t\) over a minibatch to compute the joint probability distribution \(p(u_t, a_t)\).
For the dynamics model, we use the ground-truth dynamics. We use a two-layer multi-layer perceptron for the policy and a learned baseline.
We used JAX~\cite{bradburyJAXComposableTransformations2020} for all experiments.
Additional hyperparameters are in section \ref{sec:addit-exper-deta}.

\textbf{Results: } The policy learned under the privacy constraint does indeed exactly recover the globally optimal privacy-constrained policy described in the section above.
  This policy activates the VPN on the first timestep and then always connects to the same mirror, regardless of \(u_t\), so that the actions are independent of \(u_t\).
  The non-constrained policy always chooses the mirror which corresponds to the user's IP address, resulting in a \((u_t, a_t)\)-distribution where \(u_t\) is disclosed by \(a_t\).
  The average value of \(I(a_t;u_t)\) over the episode is 1.38 (i.e. \(\log 1 / 4 \approx 1.39\)) for the non-constrained policy and 0.0047 for the constrained policy.
  Full trajectory samples are shown in figure \ref{fig:tabular-result}. This experiment illustrates that the learned constrained policies do indeed reduce \(I(a_t;u_t)\), taking pre-emptive actions in order to maximize reward under the constraint.
\vspace{-0.2cm}  
\subsection{Comparison to Differentially Private Policies}
\begin{figure}\centering
  \includegraphics[width=0.48\textwidth]{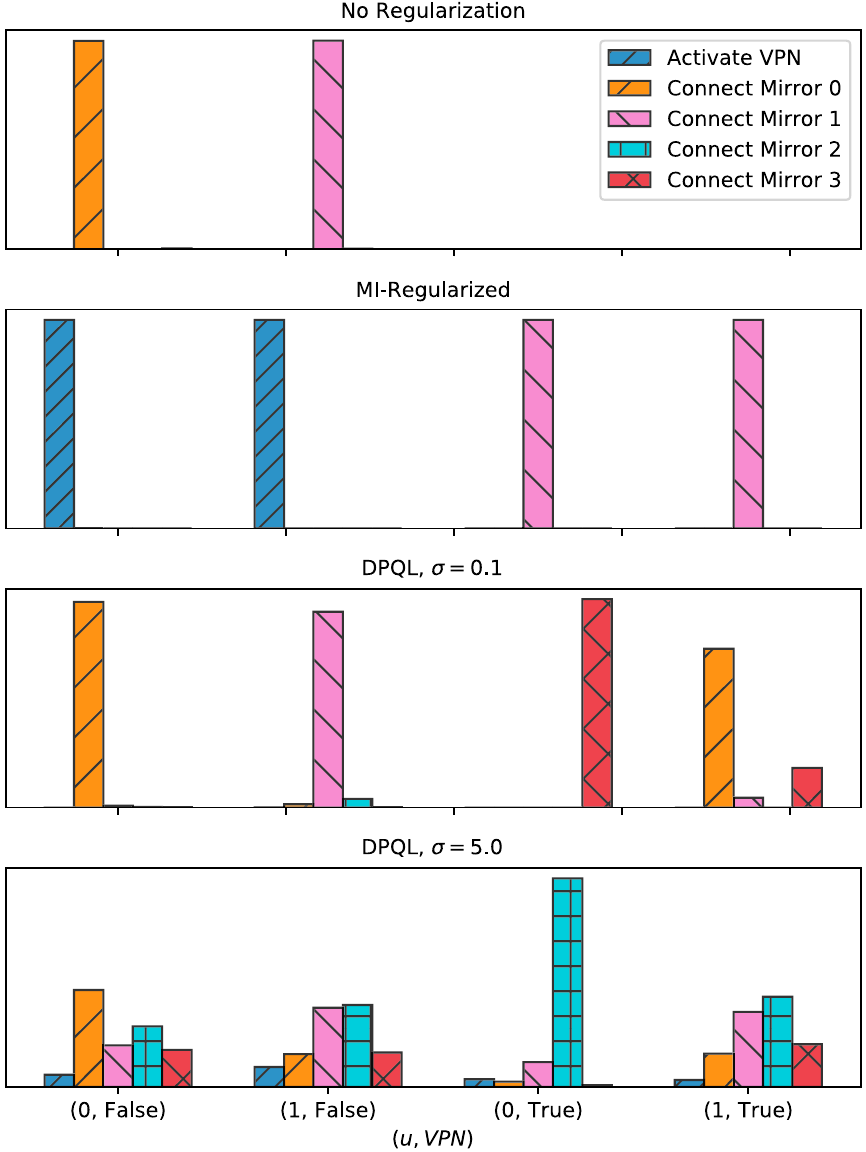}
  \caption{Action distribution in four states in the VPN MDP with four mirrors. Top to bottom: our approach with \(\boldsymbol{\lambda}=0\), \(\boldsymbol{\lambda}=1\), DPQL with \(\sigma=0.1\), \(\sigma=5.0\).}
\label{fig:dpql_actions}
\end{figure}
  We compare the policies learned with our mutual information constraint to an \((\epsilon, \delta)-\)differentially private Q-learning policy obtained via the DPQL algorithm described in \cite{wangPrivacypreservingQLearningFunctional2019}using their implementation\footnote{Found at \url{https://github.com/wangbx66/differentially-private-q-learning}}.
The provided implementation only supports a one-dimensional continuous state, so we reparameterize the VPN environment described in section \ref{sec:priv-intern-conn} with \(n=4\) mirrors, mapping the eight possible states (\(u=\{0, 1, 2, 3\}\), VPN=\(\{0, 1\}\)) to eight equal sub-regions of the interval \([0, 1]\).
The inputs to DPQL are a differential privacy budget \(\epsilon\) and a noise level \(\sigma\). As long as \(\sigma \geq \theta\), the mechanism is then \((\epsilon, \delta)\)-DP (in the sense described in section \ref{sec:demographic-parity}) where \(\theta, \delta\) depend on the batch size, learning rate, Lipschitz constant of the value approximator and other parameters. We compare several values of \(\sigma\) with \(\epsilon = 0.05\) against our method.

\textbf{Results}
A plot of the frequency of actions chosen in each state for the DPQL policy and our policy is given in figure \ref{fig:dpql_actions}.
Higher values of \(\sigma\) for DPQL lead to more noise injected into the \(Q\)-value and so a more random distribution of actions.
The higher values of \(\sigma\) for DPQL do lead to lower values of mutual information \(I(a_t;u_t)\), as the policy is more random.
However, the MI is not reduced exactly to zero. For the highest amount of noise, \(\sigma = 5.0\), the MI was \(0.11\) with reward \(7.5 \pm 0.3\), while our policy achieves a MI of \(\mathbf{0.004}\) with reward \(\mathbf{8.10 \pm 0.01}\) (the optimal reward under the mutual information constraint).
This is expected, since the DPQL approach is not aimed at satisfying a MI constraint. However, this does illustrate that the existing DP formulation is not especially suited to the problem of minimizing probabilistic disclosure of sensitive state variables. Our approach is able to take the feedback of the system into account and take preventative action to preserve privacy, while the DPQL approach simply adds noise to the policy.

\subsection{Comparison to MIRL}
We additionally compare to the MIRL approach introduced in \citet{grau2018soft}. MIRL finds a policy $\pi$ maximizing $\mathbb{E}_\pi [r] - I(s_t;a_t)/\beta$, i.e. regularizing the MI between all states and actions. In the original work the regularization is reduced to zero during training, which results in the optimal non-private policy being found. However, by fixing $\beta$, we can find an MI-regularized policy. We fix $\beta = 0.1$ and use a publicly-available re-implementation \footnote{\url{https://github.com/lcalem/reproduction-soft-qlearning-mutual-information}}. 

\textbf{Results}
The converged policy has a reward of $6.3 \pm 0.2 $ and a mutual information \(I(a_t;s_t)\) of $0.00049 \pm 0.0001$. As expected, the policy has a very low mutual information. However, it obtains a lower reward than our method. By inspection of trajectories, we observe that the final policy always selects a particular mirror (e.g. mirror 0), independently of the state. Analytically, this policy has a reward of 6.25 and an MI of 0, very close to that observed experimentally. This is lower than that obtained by our method, which first activates the VPN and then chooses arbitrarily, with total reward \(8.1\). Because MIRL constrains the mutual information between \emph{all} states and actions, our method's optimal policy is not chosen, as it has mutual information between the actions and the part of the state which denotes whether the VPN is active. This is acceptable for our approach since we have designated that part to be nonprivate. 

\subsection{Private Control in a Continuous Domain}
The second experiment is on a two-dimensional continuous control domain, illustrating the use of a learned discriminator  \(p_\psi(u_t|a_t)\), and an environment where \(u_t\) changes over the episode.
The agent controls a particle under Newtonian dynamics. The state is the coordinate positions \(x\) and \(u\), and velocities, \((x, \dot x, u, \dot u)\), the agent's actions impose a unit impulse in one of the four cardinal directions, and the reward is equal to \(-(x^2 + u^2)\).
We consider the state variable \(u\) to be sensitive. At each timestep a random isotropic Gaussian force is applied.
For this experiment we use a learned model to predict \(u_t\) given \(a_t\), an MLP parameterizing a Gaussian, conditional on the action at timestep \(t\).
Parameters of the predictor are updated on each iteration.
The prediction \(p(u_t)\) is a Gaussian at each timestep with the empirical moments of the sampled minibatch's trajectories.
The policy and critic architecture is the same as the previous experiment. Hyperparameters are in the appendix, section \ref{sec:addit-exper-deta}.

 \begin{figure}[h]
   \centering
   \hspace{-0.5cm}
   \vspace{-0.3cm}
  \begin{minipage}{0.25\textwidth}
    \includegraphics[width=\textwidth]{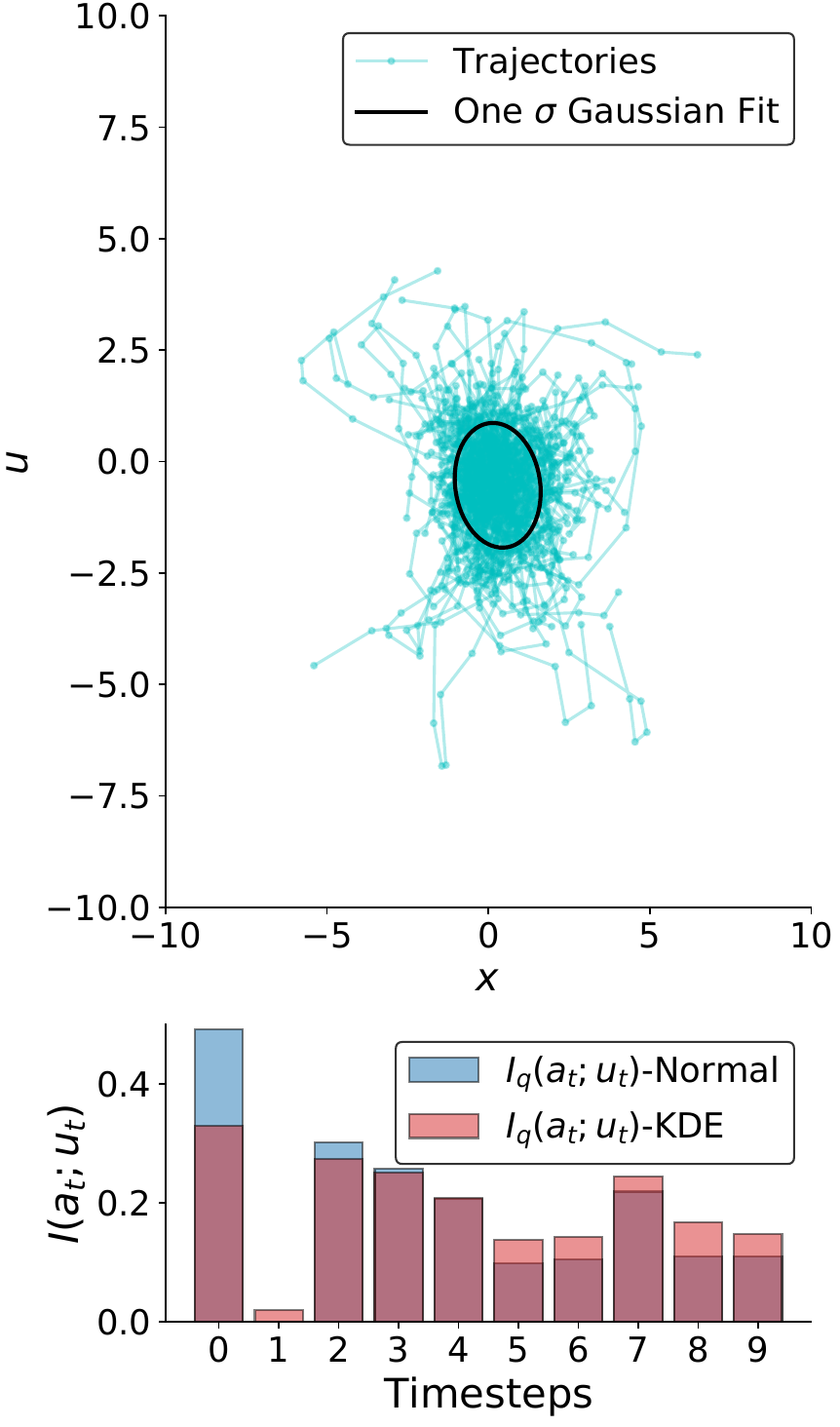}    
  \end{minipage}%
  \begin{minipage}{0.25\textwidth}
    \includegraphics[width=\textwidth]{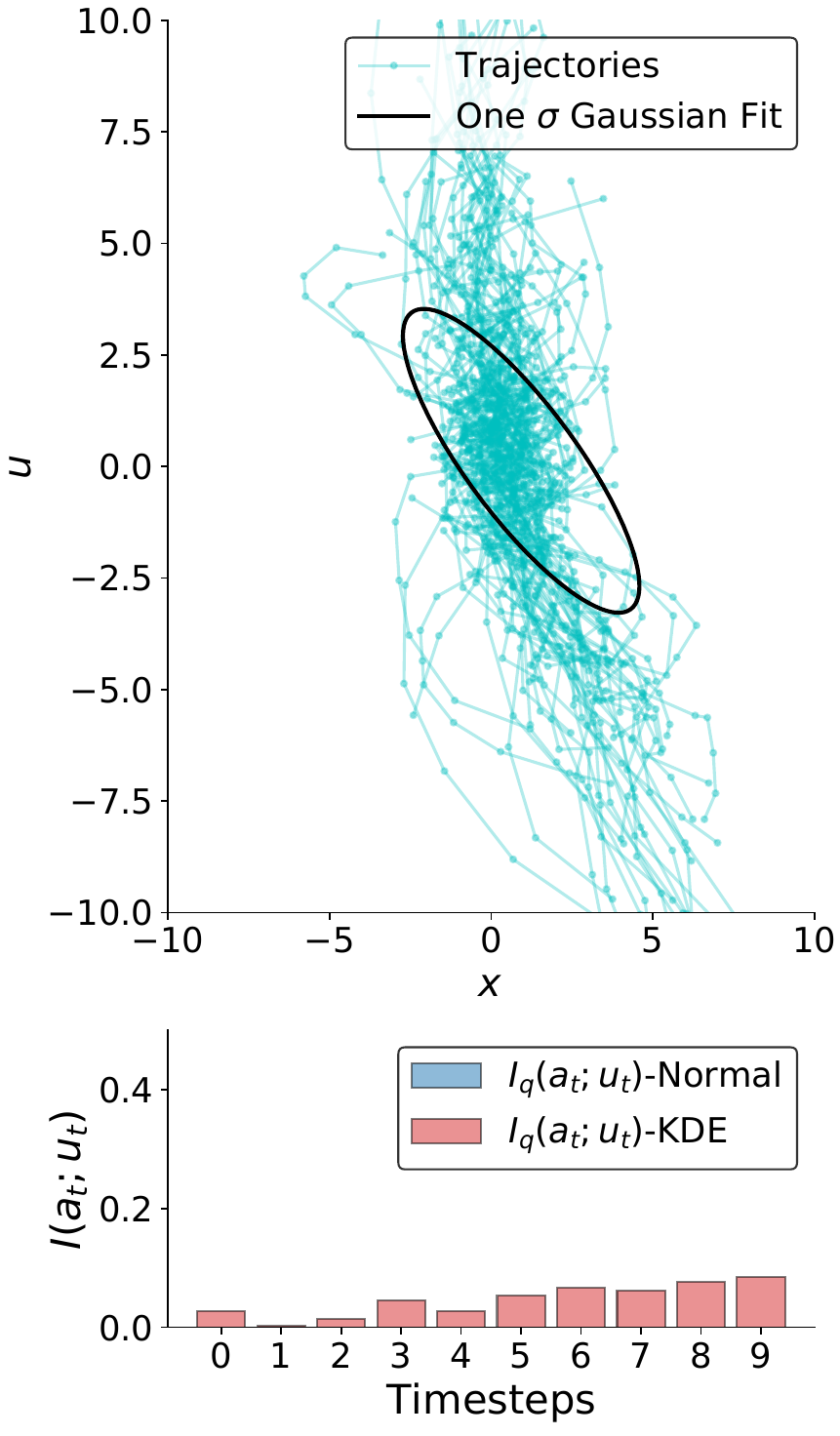}
  \end{minipage}
\caption{Trajectories for the 2d control task, \(u\)-unconstrained (left) and -constrained (right). The policy induces more variance in the \(u\)-direction in the constrained case, with less mutual information between \(a_t\) and \(u_t\). Our policy reduces the MI to zero when computed with a Gaussian discriminator, but this diverges from the MI as estimated by a nonparametric KDE at later timesteps as \(u\) is less Gaussian.}\label{fig:dynamics-result}  
\end{figure}

\begin{figure*}[ht]
   \centering
   \hspace{-0.5cm}
  \begin{minipage}{\textwidth}
    \includegraphics[width=\textwidth]{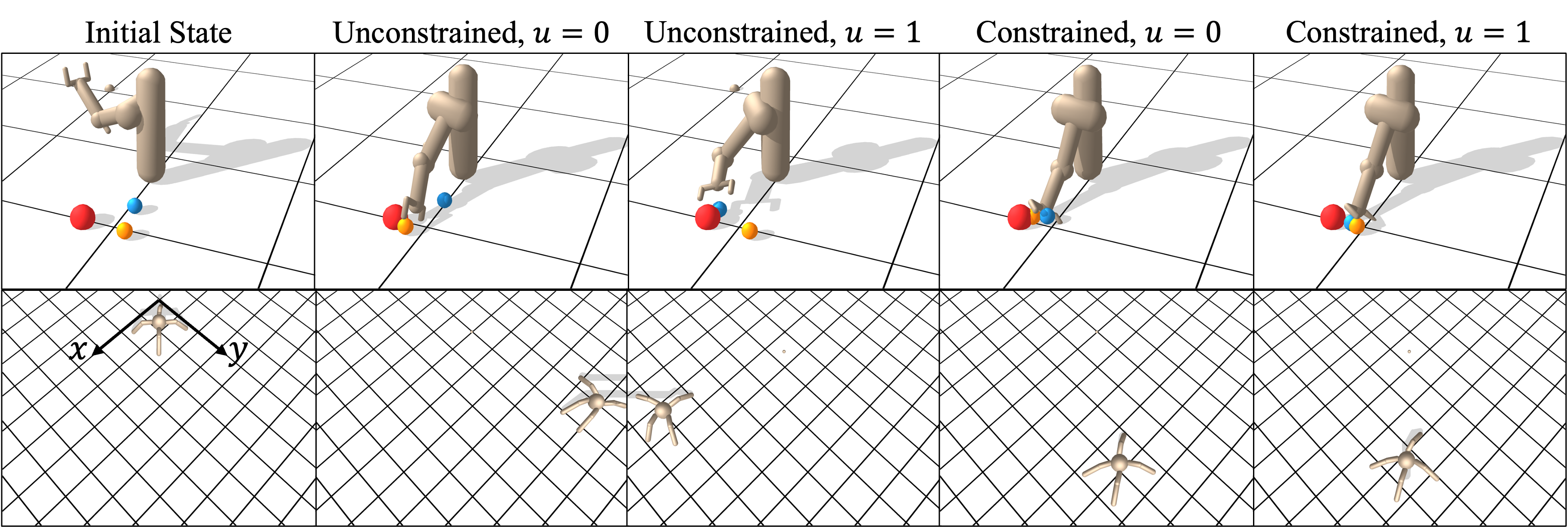}    
  \end{minipage}
  \caption{Simulated robotics policies trained with mutual information constraints. In \texttt{multi-pusher} (upper row), the goal is red; movable balls are cyan and orange. The unconstrained policy moves the active ball to the goal, while the constrained policy moves both balls to the goal. For \texttt{turning-ant} (lower row), the private policy moves diagonally while the unconstrained policy moves exactly in the direction of highest reward.}\label{fig:pusher-ant-qualitative}  
\end{figure*}

 \textbf{Results: }
The results are shown in figure \ref{fig:dynamics-result}. Trajectories from the \(u\)-constrained policy have a much higher variance in the \(u\)-direction, with correspondingly lower mutual information. Trajectories from the policy without the \(u\)-constraint have approximately equal variance in both directions.
The policy clearly trades off reward in order to satisfy the privacy constraint, as the constrained trajectories are on average further away from the center.
We also plot \(I(a_t;u_t)\), first assuming a Gaussian distribution (as is used for the discriminator in training), and then using a kernel density estimator, with no Gaussianity assumption. We see that our policy learns to reduce the mutual information under a Gaussian assumption to zero, but the true MI as measured by the KDE starts to increase at later timesteps, as the distribution of \(u\) becomes more non-Gaussian. 
\subsection{Control in a Differentiable Rigid-Body Simulator}
Here we show that our reparameterized method can be combined with modern RL algorithms such as PPO \citep{schulman2017proximal} on simulated robotics tasks with the Brax differentiable simulator \cite{freeman2021brax} to train complex policies which can hide sensitive states.
For the predictor \(p_\psi(a_{1:t};u_t)\) we use a transformer \cite{vaswani2017attention}. 
We find that the gradient norm grows with larger \(t\); to reduce it we use a surrogate predictor \(p_\psi(a_{t-k:t};u_t)\) during training, predicting \(u_t\) given the \(k\) previous actions. The predictor is a single model with \(t\)-specific positional embeddings and is trained alongside the policy. Architectural details and hyperparameters are given in section~\ref{sec:addit-exper-deta}. In this setting, instead of adjusting the Lagrange multiplier \(\lambda\) by coordinate descent, we use a PID controller, due to its success in constrained RL~\citep{stooke2020responsive}. The PID controller adjusts \(\lambda\) during training to satisfy the constraint \(\frac{1}{T}\sum_{t=1}^T I(a_{t-k:t};u_t) < \epsilon\). It is updated after each gradient step.

When computing the PPO loss, we subsample from the current batch of trajectories to obtain initial states, and then roll out \(k\) environment steps with the policy, adding the approximate MI multiplied by \(\lambda\) to the PPO loss. Gradients are computed via automatic differentiation, using Brax to allow automatic differentiation through the environment dynamics.

\begin{figure}[h!]
   \centering
  \begin{minipage}{0.25\textwidth}
    \includegraphics[width=\textwidth]{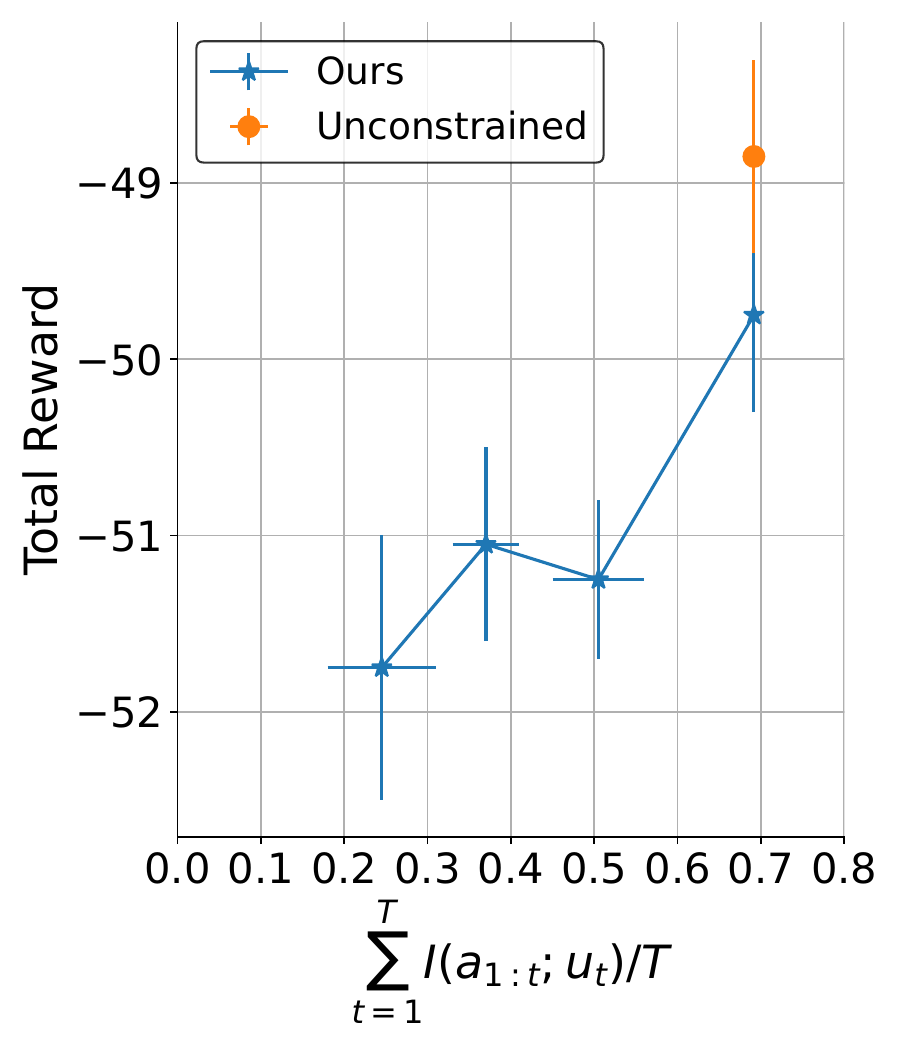}
  \end{minipage}%
  \begin{minipage}{0.25\textwidth}
    \includegraphics[width=\textwidth]{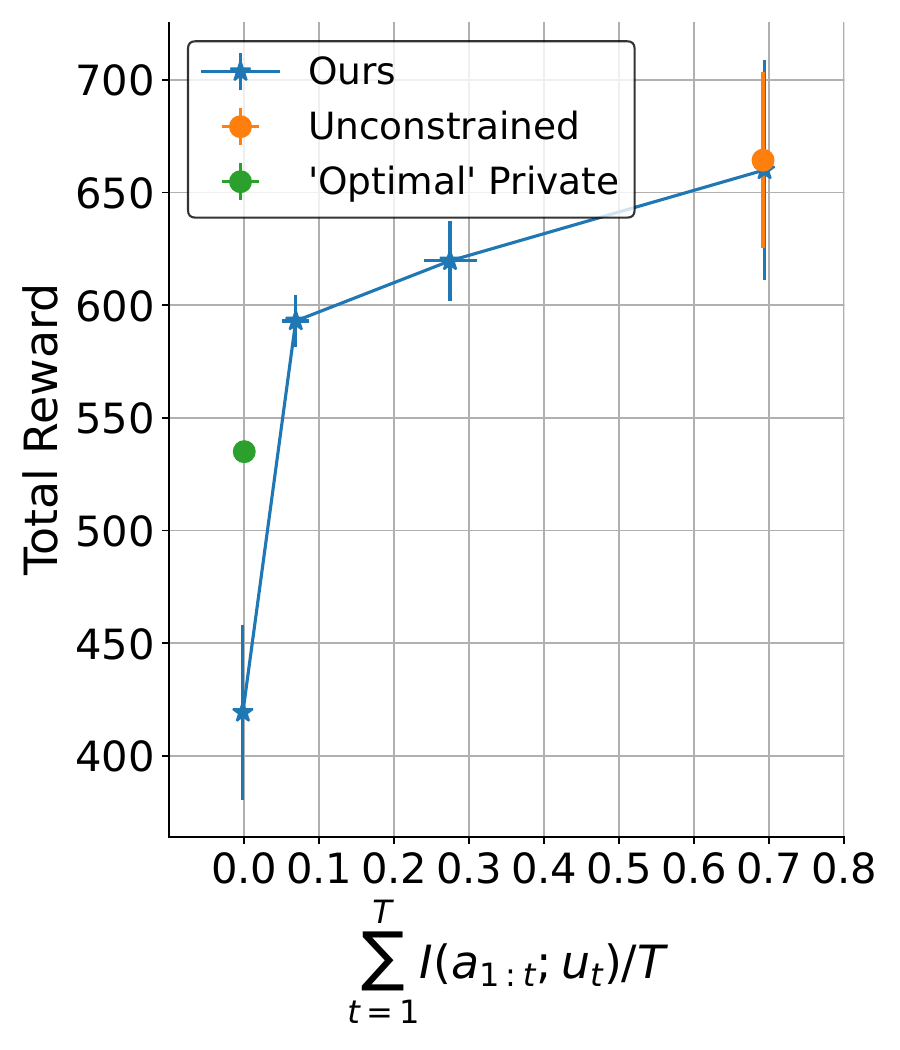}
    \end{minipage}
    \caption{Left: \texttt{Multi-Pusher}, Right: \texttt{Turning-ant}. Full-trajectory MI and reward for different levels of truncated MI constraint. In both MDPs, we find policies that reduce the disclosure of sensitive variables. In \texttt{multi-pusher}, we can reduce disclosure significantly with a minor drop in reward. In \texttt{turning-ant}, there is a trade-off between disclosure and reward.}
    \label{fig:pusher-ant-mi-reward}  
\end{figure}

We explore two classic environments, augmented with a new hidden state. First, the \texttt{multi-pusher} environment is similar to the \texttt{pusher} environment~\citep{todorov2012mujoco}, where a robot arm must move a ball to a goal. We add an additional ball and an additional binary observation which is randomly sampled at the start of the episode. This observation indicates which ball is `active'. The reward is calculated with respect only to the active ball.
Therefore, a naive policy would likely reveal which ball is the active ball, whereas hiding which ball is active requires a more sophisticated policy.

The \texttt{turning-ant} is an adaptation of the \texttt{ant} environment, where at the start of the episode we choose  to reward progress in either the \(x\) or \(y\) direction, with a corresponding binary observation. The unconstrained policy will reveal which direction is active, while a constrained policy must avoid disclosing this information.

\textbf{Results: }
Numerical results are shown in figure \ref{fig:pusher-ant-mi-reward}, for several different MI budgets \(\epsilon\). The constrained policies are able to effectively hide the sensitive state while retaining high reward. Furthermore, the truncated predictor is able to serve as an effective surrogate for the predictor trained on the whole trajectory.

Qualitative behavior of the policies is shown in figure \ref{fig:pusher-ant-qualitative}. In \texttt{multi-pusher}, the unconstrained policy pushes the active ball to the goal, revealing that this ball is active. In contrast, the constrained policy moves both balls to the goal. Figure \ref{fig:pusher-ant-mi-reward} shows this behavior can be achieved with a relatively small reduction in achieved reward. However, under the strictest truncated MI constraint, we do not achieve a full-trajectory MI of zero, illustrating potential limits of the truncated approach.

For \texttt{turning-ant}, we see that the unconstrained policy moves solely in the `active' direction, thereby revealing the identity of the sensitive parameter. In this case, we can conjecture that the optimal policy which does not reveal the sensitive parameter is to move in a diagonal direction. As shown in figure \ref{fig:pusher-ant-qualitative}, this is the policy which is found by our method. We observe that the trajectory actually arcs slightly towards the `active' direction as the episode progresses. This is due to the use of non-zero \(\epsilon\) constraint, and the difference between the truncated mutual information from \(p_\psi(a_{t-k:t};u_t)\) compared to the evaluation on \(p_\psi(a_{1:t};u_t)\). Since the policy is allowed \(\epsilon\) mutual information over a period of \(k\) steps as opposed to over \(t\) steps, the arc increases over the course of an episode for higher values of \(\epsilon\). In figure \ref{fig:pusher-ant-mi-reward} we show that the ant is able to achieve high reward while obscuring the hidden state. The `optimal' private reward is the theoretical reward if the ant were to move precisely diagonally at the same speed as in the unconstrained case. 

In section \ref{sec:control-variance} we show that the truncated predictor dramatically reduces the gradient norm when compared with the full-trajectory predictor, indicating that the truncated predictor is necessary for stable training.
\section{CONCLUSION}
\label{sec:discussion}
By minimizing the mutual information between sensitive state variables and actions, we can learn policies whose actions do not reveal the value of sensitive state variables, even in high dimensions and complex environments. Further development of this work could lead to increased trust between RL systems and users.

\section{ACKNOWLEDGEMENTS}
This research was supported by funding from the following: Stanford HAI, NSF(\#1651565), ARO (W911NF-21-1-0125), ONR (N00014-23-1-2159), and the CZ Biohub. We thank Jiaming Song, Daniel Levy, Kristy Choi and Andy Shih for valuable discussions and feedback on this research direction.

\bibliography{./bibliography}
\bibliographystyle{./icml2020}

\newpage
\section*{Checklist}
 \begin{enumerate}

 \item For all models and algorithms presented, check if you include:
 \begin{enumerate}
 \item A clear description of the mathematical setting, assumptions, algorithm, and/or model. [Yes]
 \item An analysis of the properties and complexity (time, space, sample size) of any algorithm. [No]
 \item (Optional) Anonymized source code, with specification of all dependencies, including external libraries. [Yes]
 \end{enumerate}

 \item For any theoretical claim, check if you include:
 \begin{enumerate}
 \item Statements of the full set of assumptions of all theoretical results. [Yes]
 \item Complete proofs of all theoretical results. [Yes]
 \item Clear explanations of any assumptions. [Yes]     
 \end{enumerate}

 \item For all figures and tables that present empirical results, check if you include:
 \begin{enumerate}
 \item The code, data, and instructions needed to reproduce the main experimental results (either in the supplemental material or as a URL). [Yes]
 \item All the training details (e.g., data splits, hyperparameters, how they were chosen). [Yes]
 \item A clear definition of the specific measure or statistics and error bars (e.g., with respect to the random seed after running experiments multiple times). [Yes]
 \item A description of the computing infrastructure used. (e.g., type of GPUs, internal cluster, or cloud provider). [Yes]
 \end{enumerate}

 \item If you are using existing assets (e.g., code, data, models) or curating/releasing new assets, check if you include:
 \begin{enumerate}
 \item Citations of the creator If your work uses existing assets. [Yes]
 \item The license information of the assets, if applicable. [Yes]
 \item New assets either in the supplemental material or as a URL, if applicable. [Yes]
 \item Information about consent from data providers/curators. [Not Applicable]
 \item Discussion of sensible content if applicable, e.g., personally identifiable information or offensive content. [Not Applicable]
 \end{enumerate}

 \item If you used crowdsourcing or conducted research with human subjects, check if you include:
 \begin{enumerate}
 \item The full text of instructions given to participants and screenshots. [Not Applicable]
 \item Descriptions of potential participant risks, with links to Institutional Review Board (IRB) approvals if applicable. [Not Applicable]
 \item The estimated hourly wage paid to participants and the total amount spent on participant compensation. [Not Applicable]
 \end{enumerate}

 \end{enumerate}

\newpage
\onecolumn
\appendix

\section{Appendix/Supplemental}
\subsection{Model-Based Gradient Estimation}
\label{sec:model-based-gradient}
If we have a model-based setup, we can estimate the gradient of the mutual information
constraint \(I_{q_\phi}(a_t;u_t)\) directly as follows, without introducing any \(x\) terms. We define
the approximate mutual information objective \(R(\psi)\). 
\begin{align}
  R(\psi) = \mathbb{E}_{a_t, u_t \sim q_\phi}\left[ \log p_\psi(u_t|a_t) - \log p_\psi(u_t) \right]. 
\end{align}
Now the derivative can be computed as
\begin{align}
  \nabla_\phi R(\psi) = \mathbb{E}_{a_t, u_t \sim q_\phi}\left[(\log p_\psi(u_t|a_t) - \log p_\psi(u_t))\nabla_\phi \log q_\phi(a_t, u_t)\right],
\end{align}
but we now have an issue where it's not clear how to obtain \(\nabla_\phi \log q_\phi(a_t, u_t)\), as we only typically have access to
\(q_\phi(a_t|x_t, u_t)\) while \(q_\phi(a_t, u_t) = \int_{x_t}q_\phi(a_t, x_t, u_t)dx_t = \int_{x_t}q_\phi(a_t|x_t, u_t)q_\phi(x_t, u_t)dx_t\).
Differentiating through the logarithm, we get 
\begin{align}
  \nabla_\phi \log q_\phi(a_t, u_t) & = \nabla_\phi \log \int_{x_t}q_\phi(a_t, x_t, u_t)dx_t.\\
                               & = \frac{\nabla_\phi\int_{x_t}q_\phi(a_t, x_t, u_t)dx_t}{q_\phi(a_t, u_t)}  \label{eq:lots-of-fracs}.
\end{align}
Now, we evaluate the numerator and get
\begin{align}
  \nabla_\phi\int_{x_t}q_\phi(a_t, x_t, u_t)dx_t & = \int_{x_t}\nabla_\phi q_\phi(a_t, x_t, u_t)dx_t\\
                                               & = \int_{x_t}q_\phi(a_t, x_t, u_t)\nabla_\phi \log q_\phi(a_t, x_t, u_t)dx_t\\
                                               & = \mathbb{E}_{x_t \sim q_\phi(\cdot|u_t, a_t)}\left[q_\phi(u_t, a_t)(\nabla_\phi \log q_\phi(a_t|x_t, u_t) + \nabla_\phi \log q_\phi(x_t,u_t))\right],
\end{align}
and we observe that the \(q_\phi(a_t,u_t)\) term cancels with the denominator in equation \ref{eq:lots-of-fracs},
so we have
\begin{align}
  \nabla_\phi R(\psi) = \mathbb{E}_{a_t, u_t \sim q_\phi}\left[R_\psi(u_t, a_t)\mathbb{E}_{x_t \sim q_\phi(\cdot|u_t, a_t)}\left[\nabla_\phi \log q_\phi (a_t|x_t, u_t) + \nabla_\phi \log q_\phi (x_t, u_t)\right]\right]
\end{align}
Then, we want to find 
\begin{align}
  \nabla_\phi \log q_\phi(x_t, u_t) = \frac{\nabla_\phi q_\phi(x_t, u_t)}{q_\phi(x_t, u_t)}
\end{align}
with
\begin{align}
  q_\phi(x_t, u_t) = & \int_{a'_t, x_{t-1}, u_{t-1}}q_\phi(a'_t, x_t, x_{t-1}, u_t, u_{t-1})da'_{t}dx_{t-1}du_{t-1}\\
  = & \int_{a'_t, x_{t-1}, u_{t-1}}q_\phi(a'_t|x_t, u_t, u_{t-1})q_\phi(x_t, u_t|x_{t-1}, u_{t-1})q_\phi(x_{t-1}, u_{t-1})da'_{t}dx_{t-1}du_{t-1}.
\end{align}
where we use the Markov property \(q_\phi(a_t|x_{t}, x_{t-1}, u_t, u_{t-1}) = q_\phi(a_t|x_t, u_t)\). So
\begin{align}
  \nabla q_\phi(x_t, u_t) = & \nabla_\phi \int_{a'_t, x_{t-1}, u_{t-1}}q_\phi(a'_t, x_t, x_{t-1}, u_t, u_{t-1})da'_{t}dx_{t-1}du_{t-1} \\
  = & \int_{a'_t, x_{t-1}, u_{t-1}}q_\phi(a'_t, x_t, x_{t-1}, u_t, u_{t-1})\nabla_\phi \log q_\phi(a'_t, x_t, x_{t-1}, u_t, u_{t-1}) da'_tdx_{t-1}du_{t-1} \\
  = & q_\phi(x_t, u_t)\int_{a'_t, x_{t-1}, u_{t-1}}q_\phi(a'_t, x_{t-1}, u_{t-1}|x_t, u_t)\left[\nabla_\phi \log q_\phi(a'_t|x_t, u_t) \right.\\
                          & \left. + \nabla_\phi \log q_\phi(x_t, u_{t-1}|x_{t-1},u_t) + \nabla_\phi \log q_\phi(x_{t-1}, u_{t-1})\right] da'_tdx_{t-1}du_{t-1}.
\end{align}
So
\begin{align}
  \nabla_\phi \log q_\phi(x_t, u_t) = & \int_{a'_t, x_{t-1}, u_{t-1}}q_\phi(a'_t, x_{t-1}, u_{t-1}|x_t, u_t)\left[\nabla_\phi \log q_\phi(a'_t|x_t, u_t) \right.\label{eq:cancel-policy}\\
                                 & \left. + \nabla_\phi \log q_\phi(x_t, u_t|x_{t-1},u_{t-1}) + \nabla_\phi \log q_\phi(x_{t-1}, u_{t-1})\right] da'_tdx_{t-1}u_{t-1}\\
  = & \mathbb{E}_{a'_t, x_{t-1}, u_{t-1} \sim q_\phi(\cdot, \cdot|x_t, u_t)}\left[\nabla_\phi \log q_\phi(x_t,u_t|x_{t-1},u_{t-1}) + \nabla_\phi \log q_\phi(x_{t-1}, u_{t-1})\right],\\
  = & \mathbb{E}_{x_{t-1}, u_{t-1} \sim q_\phi(\cdot|x_t, u_t)}\left[\nabla_\phi \log q_\phi(x_t, u_t|x_{t-1},u_{t-1}) + \nabla_\phi \log q_\phi(x_{t-1}, u_{t-1})\right],\\    
\end{align}
where the first term on line \ref{eq:cancel-policy} is zero by the fact that \(\mathbb{E}_{a \sim q_\phi}\left[\nabla_\phi \log q_\phi(a)\right] = 0\).
Now we have
\begin{align}
  q_\phi (x_t, u_t|x_{t-1}, u_{t-1}) = \int_{a'_{t-1}}p(x_t, u_t|a'_{t-1}, x_{t-1}, u_{t-1})q_\phi(a'_{t-1}|x_{t-1}, u_{t-1})da'_{t-1}.
\end{align}
So
\begin{align}
  \nabla_\phi \log q_\phi (x_t|x_{t-1}, u_{t-1}) = & \frac{1}{q_\phi(x_t, u_t|x_{t-1}, u_{t-1})} \int_{a'_{t-1}}p(x_t, u_t|a'_{t-1}, x_{t-1}, u_{t-1})\\& \qquad \cdot q_\phi(a'_{t-1}|x_{t-1}, u_{t-1})\nabla_\phi \log q_\phi(a'_{t-1}|x_{t-1}, u_{t-1})da'_{t-1}\\
  = & \frac{1}{q_\phi(x_t, u_t|x_{t-1}, u_{t-1})} \mathbb{E}_{a'_{t-1}\sim q_\phi (\cdot|x_{t-1}, u_{t-1})}\left[p(x_t, u_t'|a'_{t-1}, x_{t-1}, u_{t-1}) \right. \\
  & \qquad \left. \cdot \nabla_\phi \log q_\phi(a'_{t-1}|x_{t-1}, u_{t-1})\right].
\end{align}

So our expression for the gradient is
\begin{align}
  & \nabla_\phi R(\psi) = \\
  & \mathbb{E}_{a_t, u_t \sim q_\phi}\left[R_\psi(u_t, a_t)\mathbb{E}_{x_t \sim q_\phi(\cdot|u, a_t)}\left[
    \nabla_\phi\log q_\phi(a_t|x_t, u_t)\right.\right. \\
  & \left. \left. + \mathbb{E}_{x_{t-1}, a_{t-1}, u_{t-1} \sim q_\phi(\cdot, \cdot|x_t, u_t)}\left[\frac{p(x_t, u_t|a_{t-1}, x_{t-1}, u_{t-1})}{q_\phi(x_t, u_t|x_{t-1}, u_{t-1})}\nabla_\phi \log q_\phi(a_{t-1}|x_{t-1}, u_{t-1}) + \nabla_\phi \log q_\phi(x_{t-1}, u_{t-1})\right]\right]\right].
\end{align}

By repeating the decomposition we have
\begin{align}
  & \nabla_\phi \mathbb{E}_{a_t, ut_t \sim q_\phi}\left[\log p_\psi(u_t|a_t) - \log p(u_t)\right] = \\
  & \mathbb{E}_{a_t, u_t \sim q_\phi}\left[R_\psi(u_t, a_t)\mathbb{E}_{x_t \sim q_\phi(\cdot|u_t, a_t)}\left[
    \nabla_\phi\log q_\phi(a_t|x_t, u_t)\right.\right. \\
  & \left. + \mathbb{E}_{\tau_{x_{1:t-1}}, \tau_{u_{1:t-1}}, \tau_{a_{1:t-1}}\sim q_\phi(\cdot, \cdot|x_t, u_t)}\left[ \right.\right. \\
  & \left. \left. \sum_{t'=1}^{t'=t-1}\frac{p(x_{t'+1}, u_{t' + 1}|a_{t'}, x_{t'}, u_{t'})}{q_\phi(x_{t'+1}, u_{t'+1}|x_{t'}, u_{t'})}\nabla_\phi \log q_\phi(a_{t'}|x_{t'}, u_{t'})\right]\right],
\end{align}

What do we need to compute this gradient estimator?
\begin{itemize}
\item{We need to be able to sample from \(\tau_{x_{1:t-1}}, \tau_{a_{1:t-1}}, \tau_{u_{1:t-1}}\), which we can get from trajectory samples.}
\item{We need to be able to compute \(q_\phi(x_t, u_t|x_{t-1}, u_{t-1})\), and \(p(x_t, u_t|a_{t-1}, x_{t-1}, u_{t-1})\).
    In practice we can use \(p(x_t, u_t|a_{t-1}, x_{t-1}, u_{t-1})\) to compute \(q_\phi(x_t, u_t|x_{t-1}, u_{t-1})\) because
  \(q_\phi(x_t, u_t|x_{t-1}, u_{t-1}) = \mathbb{E}_{a_{t-1}\sim q_\phi(\cdot|x_{t-1}, u_{t-1})}\left[p(x_t, u_t|x_{t-1},a_{t-1}, u_{t-1})\right]\)}
\end{itemize}

Of course a special case is \(t=1\), where we have \(\nabla_\phi p(x_1, u_1) = 0\), so
\begin{align}
  & \nabla_\phi R_1(\phi) = \\
  & \mathbb{E}_{a_1, u_1 \sim q_\phi}\left[R_\psi(u_1, a_1)\mathbb{E}_{x_1, \sim q_\phi(\cdot|u_1, a_1)}\left[\nabla_\phi\log q_\phi(a_1|x_1, u_1)\right]\right].
\end{align}

with \(R_\psi(u_1,a_1) = \left[(\log q_\psi(u_1|a_1) - \log p(u_1))\right]\).

\newpage
\subsection{\(I(\tau_a;\tau_u)\) Regularizer}
\label{sec:itau_a-tau_u-regul}
As discussed in the main body, another possible threat model is an adversary aiming to infer the whole trajectory of sensitive states \(\tau_u\)
from the whole trajectory of actions \(\tau_a\). We sketch out a basis for forming an estimator for \(\nabla_\phi I_{q_{\phi}}(\tau_a;\tau_u)\).
We want to compute
\begin{align}
  &\nabla_\phi \mathbb{E}_{\tau_a, \tau_u \sim q_\phi}\left[I_{q_{\phi}}(\tau_a;\tau_u)\right]\\
  & = \int_{\tau_a, \tau_u}\left[\log p_\psi(\tau_u|\tau_a) - \log p_\psi(\tau_u)\right] \nabla_\phi q_\phi(\tau_a, \tau_u)d\tau_ad\tau_u\\
  & = \mathbb{E}_{\tau_a, \tau_u \sim q_\phi}\left[\left(\log p_\psi(\tau_u|\tau_a) - \log p_\psi(\tau_u)\right) \nabla_\phi \log q_\phi(\tau_a, \tau_u)\right].
\end{align}
As before, the difficulty arises in computing \(\nabla_\phi q_\phi(\tau_a, \tau_u)\) which involves a marginalization over
the non-sensitive state \(\tau_x\).
  Now 
  \begin{align}
    q_\phi(\tau_a, \tau_u) = \int_{\tau_x}q_\phi(\tau_a, \tau_x, \tau_u)d\tau_x.
  \end{align}
  For conciseness, we write \(p(x_{t+1}, u_{t+1}|x_t, u_t, a_t)q_\phi(a_{t+1}|x_{t+1},u_{t+1}) = q_\phi(x_{t+1}, u_{t+1}, a_{t+1}|x_t, u_t, a_t)\).
  Also note that
  \begin{align}
    \nabla_\phi q_\phi(x_{t+1}, u_{t+1}, a_{t+1}|x_t, u_t, a_t) & = p(x_{t+1}, u_{t+1}|a_t, x_t, u_t)\nabla_\phi q_\phi(a_{t+1}|x_{t+1},u_{t+1})\\
                                                       & = q_\phi(x_{t+1}, u_{t+1}, a_{t+1}|x_t, u_t, a_t)\nabla_\phi \log q_\phi(a_{t+1}|x_{t+1},u_{t+1}).
  \end{align}
  We then have
    \begin{align}
      q_\phi(\tau_a, \tau_u) = \int_{\tau_x}q_\phi(x_T, u_T, a_T|x_{T-1}, u_{T-1}, a_{T-1})q_\phi(\tau_{a_{1:T-1}}, \tau_{x_{1:T-1}}, \tau_{u_{1:T-1}})d\tau_x,     
    \end{align}
    so
    \begin{align}
      \nabla_\phi q_\phi(\tau_a, \tau_u) = & \int_{\tau_x}q_\phi(\tau_{a_{1:T-1}}, \tau_{x_{1:T-1}}, \tau_{u_{1:T-1}}) \nabla_\phi q_\phi(x_T, u_T, a_T|x_{T-1}, u_{T-1}, a_{T-1}) \\
                                           & + q_\phi(x_T, u_T, a_T|x_{T-1}, u_{T-1}, a_{T-1}) \nabla_\phi q_\phi(\tau_{a_{1:T-1}}, \tau_{x_{1:T-1}}, \tau_{u_{1:T-1}}) d\tau_x\\
                                             =  & \int_{\tau_x}q_\phi(\tau_{a_{1:T}}, \tau_{x_{1:T}}, \tau_{u_{1:T}}) \nabla_\phi \log q_\phi(a_T|x_T, u_T) \\
                                           & + q_\phi(x_T, u_T, a_T|x_{T-1}, u_{T-1}, a_{T-1}) \nabla_\phi q_\phi(\tau_{a_{1:T-1}}, \tau_{x_{1:T-1}}, \tau_{u_{1:T-1}}) d\tau_x\\
                                           = & \ q_\phi(\tau_a, \tau_u)\mathbb{E}_{\tau_x\sim q_\phi(\cdot|\tau_a, \tau_u)}\left[\nabla_\phi \log q_\phi(a_T|x_T, u_T) \right.\\
                                           & + \left. \nabla_\phi \log q_\phi (\tau_{a_{1:T-1}}, \tau_{x_{1:T-1}}, \tau_{u_{1:T-1}})\right]\\
    \end{align}
    Now
    \begin{align}
      & \mathbb{E}_{\tau_x\sim q_\phi(\cdot|\tau_a, \tau_u)}\left[\nabla_\phi \log q_\phi (\tau_{a_{1:T-1}}, \tau_{x_{1:T-1}}, \tau_{u_{1:T-1}})\right]\\
      & \ = \mathbb{E}_{\tau_x\sim q_\phi(\cdot|\tau_a, \tau_u)}\left[\nabla_\phi \log q_\phi (\tau_{x_{1:T-1}}| \tau_{a_{1:T-1}}, \tau_{u_{1:T-1}}) + \log q_\phi (\tau_{a_{1:T-1}}, \tau_{u_{1:T-1}})\right]\\
      & \ = \mathbb{E}_{\tau_x\sim q_\phi(\cdot|\tau_a, \tau_u)}\left[\nabla_\phi \log q_\phi (\tau_{a_{1:T-1}}, \tau_{u_{1:T-1}})\right],
    \end{align}
    by the fact that \(\mathbb{E}_{a \sim q_\phi}\left[\nabla_\phi \log q_\phi(a)\right] = 0\).
    And so we have a reduction from \(\nabla_\phi q_\phi(\tau_a, \tau_u)\) to  \(\nabla_\phi q_\phi(\tau_{a_{1:T-1}}, \tau_{u_{1:T-1}})\), similarly to \(q(a_t,u_t)\) case in the section above.
    We can repeat this to form an estimator of \(\nabla_\phi I_{q_\phi}(\tau_a;\tau_u)\).

    \newpage
    \subsection{Proof for theorem \ref{sec:dual-formulation}}
    We consider an MDP where \(u_t\) is independent of the actions. This could be because the hidden state changes randomly at each timestep, or (more practically relevant) because the hidden state is randomly chosen at the start of the episode and the agent must avoid leaking the information about the fixed \(u\). For the purposes of this proof, we  assume that the constraint \(\epsilon_t\) is strictly nonzero, as otherwise Slater's condition does not hold. This is consistent with our description in e.g. equation \eqref{eq:main-problem}, where we use a strict inequality.

    As a preliminary, we note the convexity of the mutual information over the conditional distribution: if we have three distributions \(p_1(x,y) = p(x)p_1(y|x)\), \(p_2(x,y) = p(x)p_2(y|x)\), \(p_3(x,y) = p(x)(\lambda p_1(y|x) + (1 - \lambda)p_2(y|x))\), then \(I_3 \leq \lambda I_1 + (1-\lambda)I_2\) for the corresponding mutual informations. For proof, see e.g.~\citet{coverElementsInformationTheory1991}.
    
    We will follow the proof approach from \citet{paternain2019constrained}.
    First, the perturbation function \(P\) is defined as follows:
    \begin{align*}
      P(\boldsymbol{\xi}) = \underset{\phi}{\text{Max}} \ \underset{\tau \sim q_\phi}{\mathbb{E}}\left[\sum_{t=1}^Tr(x_t,u, a_t)\right], \quad \text{s.t.} \ I(a_t;u_t) < \epsilon_t -\xi_t \ \forall t.
    \end{align*}
    From \citet{paternain2019constrained}, theorem 1, an optimization problem satisfying Slater's condition and a concave perturbation function has zero duality gap. 
    Slater's condition is clearly satisfied, as there is a set of policies which have exactly zero mutual information (all random policies). To show the concavity of the perturbation function, we must show that 
    \begin{align*}
      P(\mu \xi^1 + (1-\mu)\xi^2) \geq \mu P(\xi^1) + (1-\mu)P(\xi^2).
    \end{align*}
    In other words, for a policy \(\pi_1\) that attains reward \(R_1\), the maximum for \(\xi^1\), and a policy \(\pi_2\) that attains reward \(R_2\), the maximum for \(\xi^2\), the reward-maximizing policy for constraint \(\mu\xi^1 + (1-\mu)\xi^2\) must have reward at least \(\mu R_1 + (1-\mu)R_2\). We will do this by finding a policy \(\pi_\text{mix}\) which satisfies the mutual information constraint for \(\mu\xi^1 + (1-\mu)\xi^2\)  and has reward equal to \(\mu R_1 + (1-\mu)R_2\). Note that in cases where \(\xi^1\) or \(\xi^2\) results in an infeasible constraint, \(P(\xi^1)\) or \(P(\xi^2)\) is \(-\infty\) and we satisfy the constraint.

    We now introduce the discounted occupancy measure \(\rho_t(x_t, u_t, a_t) = \gamma^t \pi_t(a_t |x_t,u_t) p_t(x_t,u_t)\). This gives the probability that at time \(t\), the agent is in state \((x_t,u_t)\), and chooses action \(a_t\).
    We can write the reward as an expectation: \(R = \mathbb{E}_{(x_t, u_t, a_t) \sim \rho_t(x_t, u_t, a_t)}\left[ r(x_t, u_t, a_t) \right]\). 
    Now, the space of occupancy measures is convex, so for two policies \(\pi_1\), \(\pi_2\), with occupancy measures \(\rho_1\), \(\rho_2\), there exists a policy \(\pi_\text{mix}\) with occupancy measure \(\rho_\text{mix} = \mu \rho_1 + (1-\mu)\rho_2\). This policy, which we denote \(\pi_\text{mix}\), achieves reward \(R_\text{mix} = \mu R_1 + (1-\mu)R_2\), due to the linearity of expectation with the expression for the reward above. To prove convexity of \(P\), all that remains is to show that \(I_{q_\text{mix}}(a_t; u_t) \leq \epsilon_t - (\mu \xi_t^1 + (1-\mu)\xi_t^2\)).

    Observe that
    \begin{align*}
      P_\text{mix}(a_t, u_t) = \int_\mathcal{X} P_\text{mix}(a_t, x_t, u_t)dx = \int_\mathcal{X} \mu P_1(a_t, x_t, u_t) + (1 - \mu) P_2(a_t, x_t, u_t)dx = \mu P_1(a_t, u_t) + (1 - \mu) P_2(a_t,  u_t).
    \end{align*}

    Now, 
    \begin{align*}
      I_{q_\text{mix}}(a_t; u_t) = \mathbb{E}_{a_t, u_t \sim q_\text{mix}}\left[\log q_\text{mix}(a_t | u_t) - \log q(u_t)\right],
    \end{align*}
    where we have written \(\log q(u_t)\) with no additional quantifiers since we are assuming that \(q(u_t)\) doesn't depend on the policy.
    As described above, the mutual information is convex in the conditional distribution, so we have

    \begin{align*}
      I_{q_\text{mix}}(a_t; u_t) \leq \mu I_1 + (1-\mu)I_2 \leq \epsilon_t -(\mu\xi^1_t + (1-\mu)\xi^2_t).
    \end{align*}
    Therefore, the policy \(p_\text{mix}\) is a feasible policy for the constraint \(\mu\xi^1_t + (1-\mu)\xi^2_t\) and achieves a reward of at least \(\mu R_1 + (1-\mu)R_2\). Therefore, the perturbation function is concave and via the result of \citet{paternain2019constrained}, the problem exhibits strong duality.

    \newpage
    \subsection{Additional Experiments}
    \subsubsection{Figures for Internet Connectivity Example}
    Trajectories for the internet connectivity example are shown in figure \ref{fig:tabular-result}.
\begin{figure*}[h]
  \begin{minipage}{0.48\textwidth}
    \includegraphics[width=\textwidth]{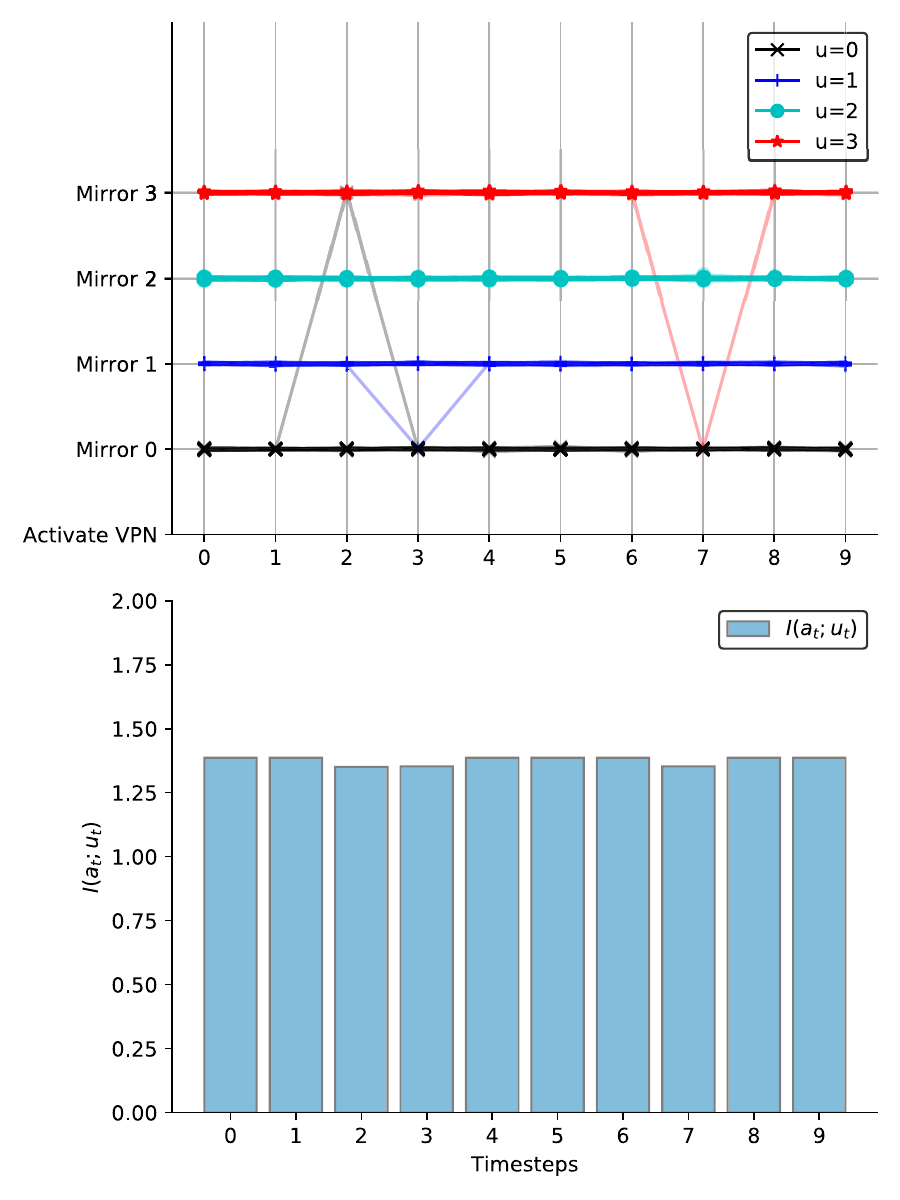}    
  \end{minipage}%
  \begin{minipage}{0.49\textwidth}
    \includegraphics[width=\textwidth]{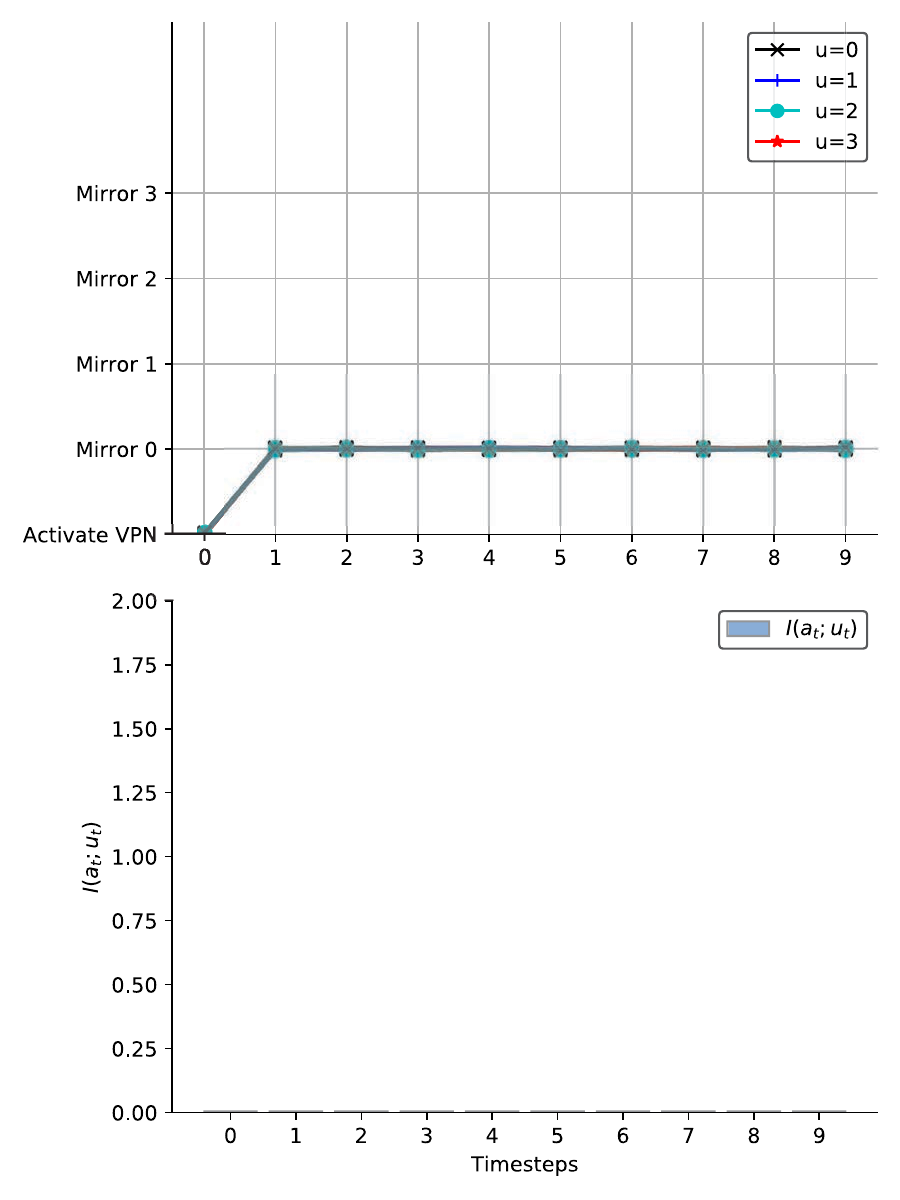}
  \end{minipage}%
  \caption{Trajectories from the internet connectivity environment. One the left we have a Lagrange multiplier \(\boldsymbol{\lambda} = \boldsymbol{0}\), while on the right we have \(\boldsymbol{\lambda} = \boldsymbol{1}\). We see that the trajectories in the constrained case are able to completely remove the mutual information between the action and the sensitive state by choosing a policy of always activating the VPN and then choosing mirror 0.}  \label{fig:tabular-result}
\end{figure*}

\newpage
\subsubsection{Controlling Individual Timesteps' Mutual Information}
\label{sec:contr-indiv-timest}
In this experiment we compare the model-based estimator and the model-free estimator, in addition to demonstrating
the control over individual timesteps' mutual information.
We consider an idealized customer-service problem. In this problem, our agent controls a personalized goods
distribution system, delivering goods to a specific person.  The person has a location \(x \in
\mathbb{R}\), obeying a random walk where \(x_{t+1} = x_t + \epsilon\), for \(\epsilon \sim \mathcal{N}(0,\sigma^2)\), and
the agent has a `service center' which is on a lattice \(w \in \mathbb{Z}\) of possible locations
from which deliveries are sent out.  At each timestep, the agent is told the location of the person,
and updates the location of its service center, either increasing it by 1 or decreasing it by 1. The
agent then receives reward \(r(x_t, a_t, w_t) = -|x_t - w_t|\), being penalized for how far the
service center is from the person. To make this a privacy-constrained problem, we suppose that
there is an underlying sensitive binary variable \(u \in \{0,1\}\), which heavily influences the
initial position of the client. The sensitive variable \(u\) is constant over the episode.
In particular, we choose \(p(x_1, u)\) as \(p(u)p(x_1|u)\)
with \(u \sim \text{Uniform}\ \{0,1\}\), \(x \sim \mathcal{N}(2u, \sigma_0^2)\). For our
experiments we used \(\sigma^2 = 1/4\) and \(\sigma_0^2 = 0.5\). Since this environment is \(u\)-shielded
per our definition in the appendix \ref{sec:u-shielded}, where \(u\) also impacts the dynamics, with the update \(x_{t+1} = x_t + u\alpha + \epsilon\). We used \(\alpha = 0.3\) in cases where we use this variant. 

\subsubsection{Results}
\textbf{Customer Service}
The results of the customer service experiments are shown in figures \ref{fig:timestep-customer}, \ref{fig:global-customer}.
The experiments using the model-based estimator in equation \ref{eq:model-gradient} in figure \ref{fig:timestep-customer}
show that the model-based method is indeed able to selectively constrain the
value of \(I(u;a_t)\) by choosing the right lagrange multipliers in equation \ref{eq:main-dual-problem}.
The trajectories match the intuitive expectations of the constraints: the unconstrained agent draws the trajectories as
close to zero as quickly as possible to maximize reward. However the constrained agent is not able to
do this, as it would reveal the protected variable. The agent that is heavily constrained on the first
timestep moves both groups down, even though this results in less reward for the blue group: once at a
timestep where it isn't constrained, it moves both groups back towards zero. 

We also present the model-free approach in figure \ref{fig:global-customer}, evaluating on the
non-\(u\)-shielded version of the customer service example. Although harder to interpret the
behaviour due to the global nature of the constraint over the whole sequence of actions, we see that
the constrained agent takes a similar approach in choosing similar distributions of actions for both
groups, while drawing the two groups gradually closer to the origin. Again, the unconstrained agent simply
draws both groups to the origin immediately. Examining the trajectory-level mutual information, we
see that the constrained agent has \(I(u;\tau_a, \tau_x) \approx 0.35\), while for the unconstrained
agent it is approximately \(0.45\). 

\begin{figure*}[h]
  \label{fig:customer-results}
  \hspace{-0.5cm}
  \begin{minipage}{0.33\textwidth}
    \includegraphics[width=\textwidth]{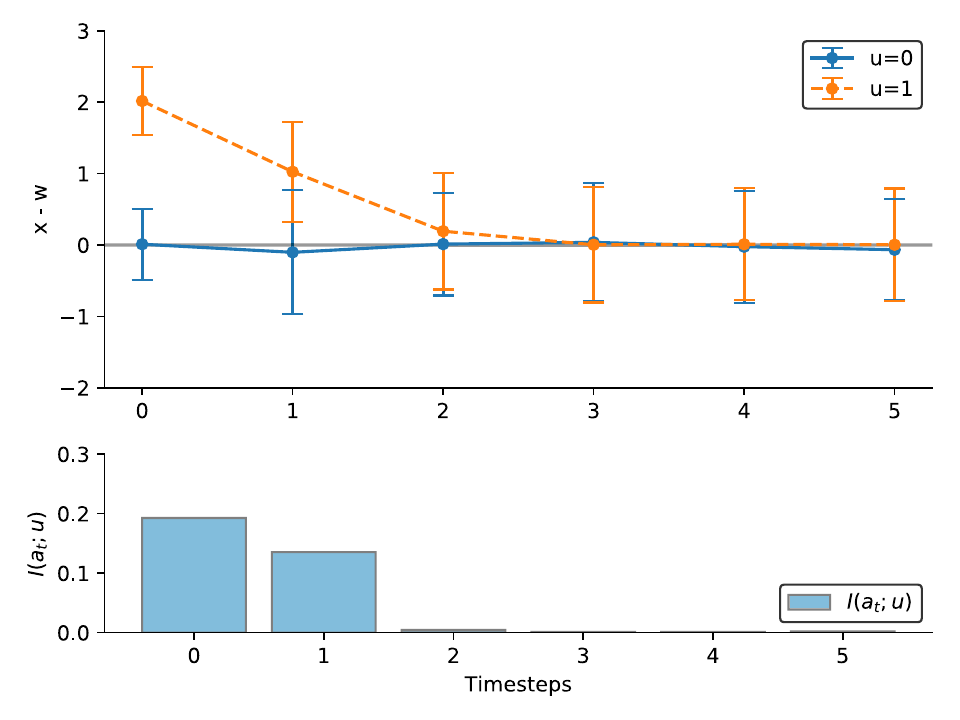}
  \end{minipage}%
  \begin{minipage}{0.32\textwidth}
    \includegraphics[width=\textwidth]{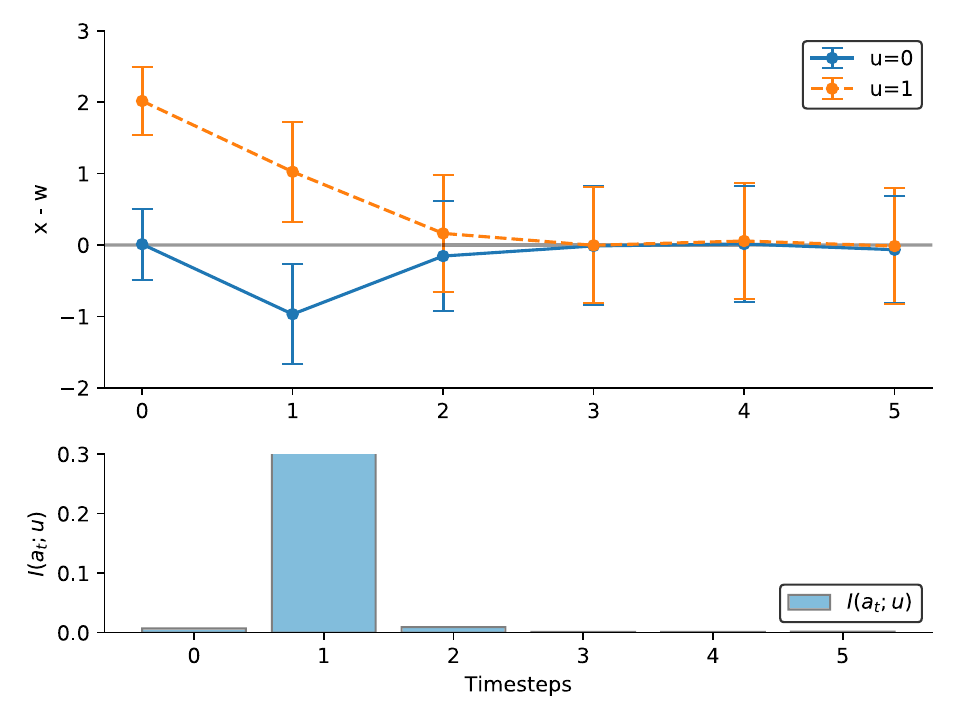}
  \end{minipage}%
  \begin{minipage}{0.33\textwidth}
    \includegraphics[width=\textwidth]{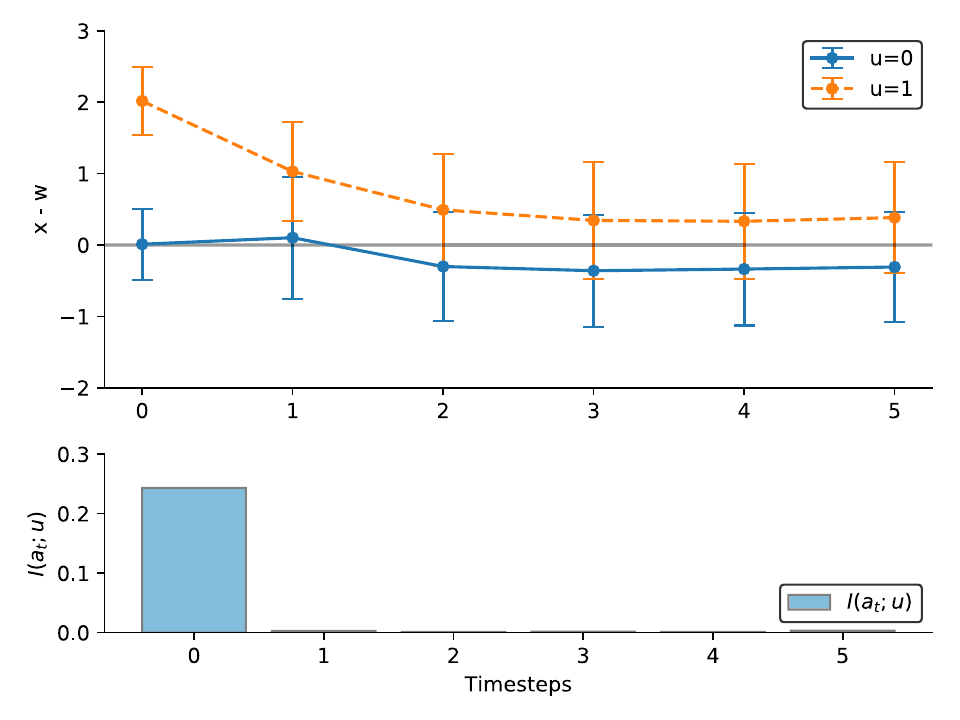}
  \end{minipage}%
  \caption{Trajectories and values of \(I(a_t;u)\) generated from different choices of Lagrange
    multipliers in equation \ref{eq:main-dual-problem} on the customer service problem, for the two
    different protected groups. The left shows
    \(\boldsymbol{\lambda} = \boldsymbol{0}\), the middle shows \(\boldsymbol{\lambda} = (10, 0, \ldots,
    0)\), and the right shows \(\boldsymbol{\lambda} = (0, 10, \ldots, 10)\). We get precise control
    over \(I(a_t;u)\) with the choices of \(\boldsymbol{\lambda}\). For example, in the middle trajectory we see
    that a large Lagrange multiplier on the first timestep's constraint forces the agent to treat both
    groups exactly the same, even at the expense of reward by moving away from zero. Once the constraint
    is removed, the agent exposes a large amount of information about the protected groups with its subsequent
  action.}
  \label{fig:timestep-customer}
\end{figure*}

\begin{figure}[h]
\centering
  \begin{minipage}{0.48\textwidth}
    \includegraphics[width=\textwidth]{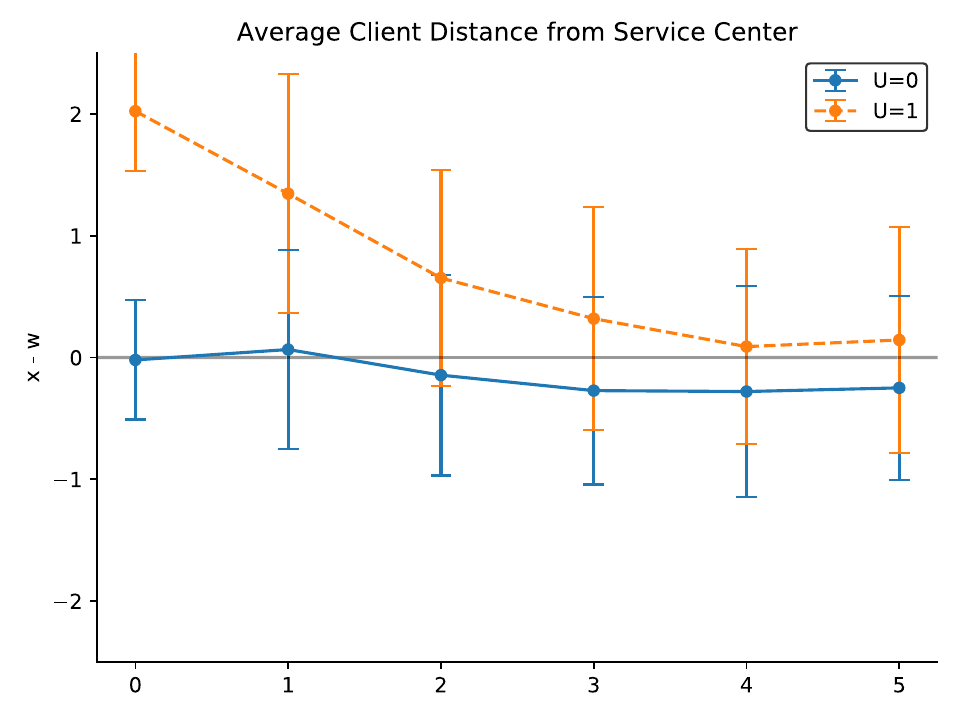}
  \end{minipage}%
  \caption{Trajectories induced by an agent trained against the model-free upper bound constraint. The
  trajectories are harder to interpret than when trained against the model-based constraint.}
  \label{fig:global-customer}  
\end{figure}

\newpage
\subsection{SNAP Allocation}
\begin{figure}
\centering
  \begin{minipage}{0.26\textwidth}
    \includegraphics[width=\textwidth]{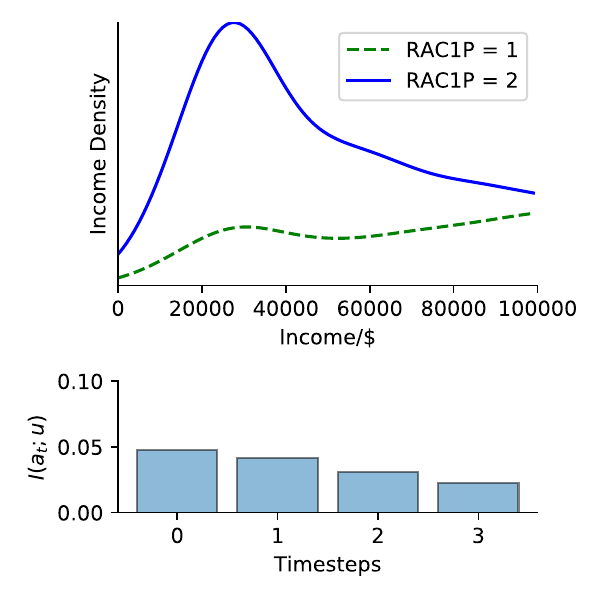}
  \end{minipage}%
  \begin{minipage}{0.26\textwidth}
    \includegraphics[width=\textwidth]{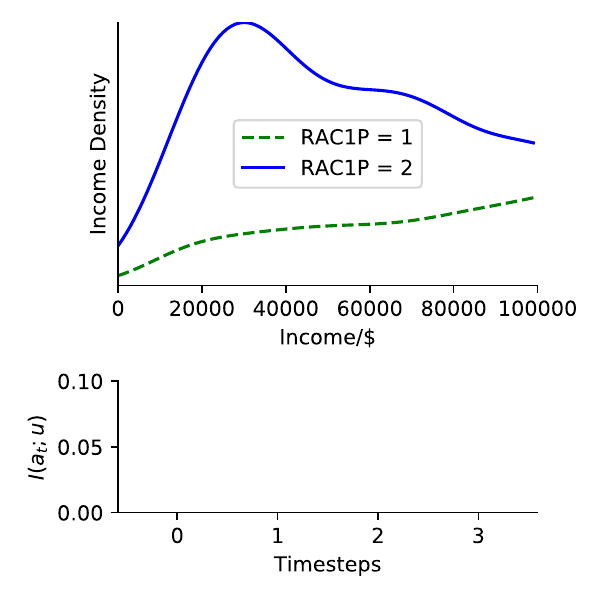}
  \end{minipage}%
  \caption{Final income distributions from privacy-constrained (right) and unconstrained (left) policies on the SNAP setting. The constrained policy removes all action-protected-variable mutual information.}
\label{fig:snap-results}  
\end{figure}

To investigate the possibility of applications in algorithmic fairness, we implement a more realistic experiment using the supplemental nutrition assistance program (SNAP) data from the American Community Survey (ACS) microdata.
For simplicity, we restrict our attention to households from Washington D.C., and to people who recorded their race under the `RAC1P'
code as `1' or `2', corresponding to \texttt{White alone} and \texttt{Black or African American
  Alone}.
This resulted in 5,304 records, of which 2,933 were white and 2,371 black.
Using these data, we then form a kernel density estimator of the income distribution of both race codes.
The agent has a binary action, to either give SNAP or not.
The next timestep is imagined to happen a year later.
The person's income is \(x_t = x_{t-1} + \epsilon\) in the case where SNAP isn't given (where \(\epsilon \sim \mathcal{N}(0, \sigma^2)\)), and in the case where SNAP is given the income is \(x_t = x_{t-1} + \epsilon + \gamma\), with \(\gamma \sim \text{Uniform}(1512, 7200)\). This \(\gamma\) represents the increase in income due to the SNAP program, where the minimum and maximum limits correspond to the upper and lower bounds of SNAP contribution~\cite{fangFoodFairnessArtificial}.
We chose \(\sigma = \) \$1000.
We treat the race variable as protected, and the income variable as unprotected.
The reward at each timestep is \(-\min\{0, L-x_t\}^2\), where \(L\) is the federal poverty level of
\$24,900.

\textbf{Results:} As we can see in figure \ref{fig:snap-results}, we are able to reduce the mutual information
exposed by the policy's actions while not changing the final income distribution unduly.
In the unconstrained case there is a noticeable peak around the poverty level as the policy
strictly gives assistance to those below and not above. The constrained case shows a gentler increase
corresponding to a stochastic policy that is limited in the extent to which it can use the sensitive race
variable (which is correlated to income).
As discussed in section \ref{sec:demographic-parity}, careful consideration of the societal
impacts of using this independence measure as a fairness benchmark is required.

\newpage
\subsection{Comparison to Differentially-Private Q-Learning}
\label{sec:comp-diff-priv}
Figure \ref{fig:dpql_actions_2} shows the action distribution for our approach, compared to the DPQL approach. We show the four states corresponding to \(u \in \{0, 1\}\) in the VPN example with four mirrors. The states with \(u \in \{2, 3\}\) are similar and omitted to make the plot clearer. We see that the DPQL approach works through adding noise to the \(Q\)-values, while our approach takes into account the feedback in the MDP, choosing the VPN at the first timestep.

\begin{figure}\centering
  \includegraphics[width=0.6\textwidth]{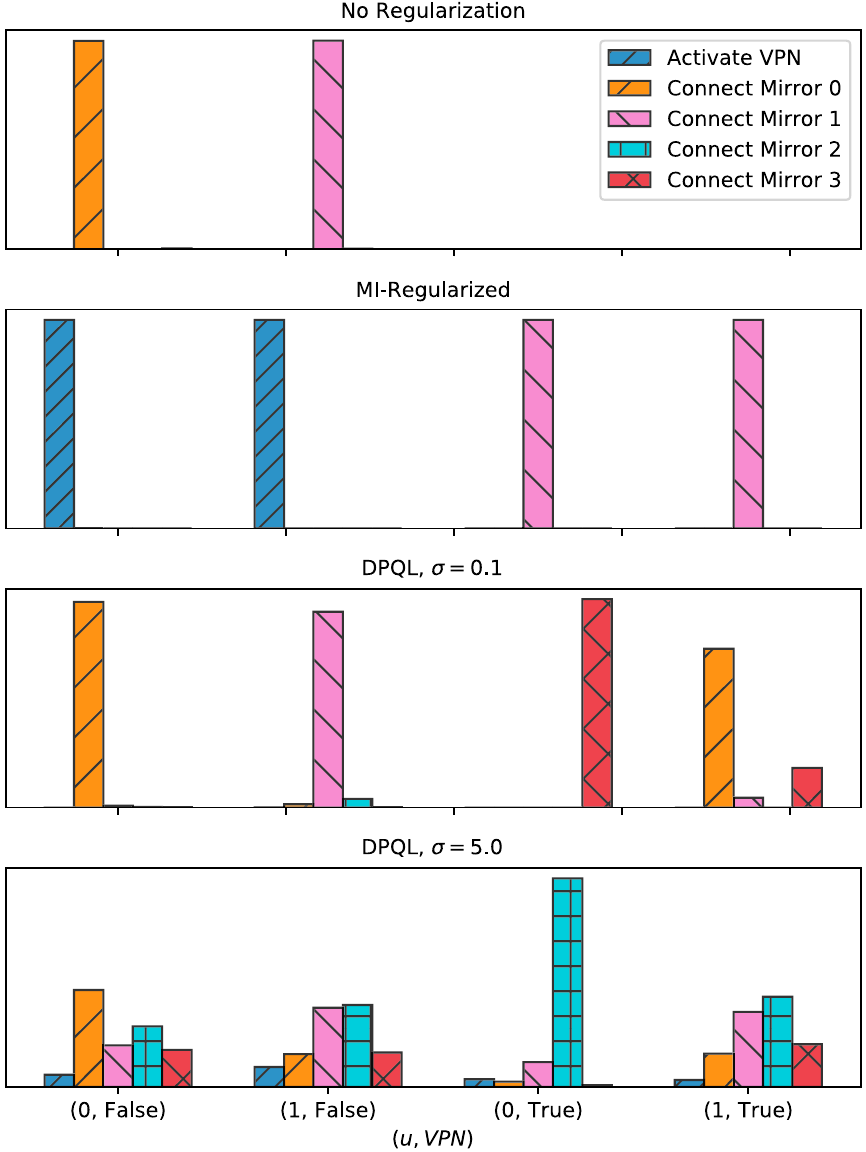}
  \caption{Distribution of actions in four states in the VPN example with four mirrors. From top to bottom, we have our approach with \(\boldsymbol{\lambda}=0\), our approach with \(\boldsymbol{\lambda}=1\), DPQL with \(\sigma=0.1\), and \(\sigma=5.0\)}
\label{fig:dpql_actions_2}
\end{figure}
 \newpage
\subsection{Full Algorithm}
\label{sec:full-algorithm}
For completeness we give the full algorithmic details in algorithms 1, 2 and 3. This pseudocode is very close to the
actual implementation.
    \begin{algorithm}[h]
      \label{alg:bedandsco}
      \caption{Model-Based MI-Constrained Policy Gradients}
      \begin{algorithmic}
        \STATE {\bfseries Input:} Vector \(\boldsymbol{\epsilon_t} \in \mathbb{R}_{\geq 0}^T\),
        number of \(\psi\) gradient steps \(n\), number of \(\phi\) gradient steps \(m\),
        gradient-based optimization method \texttt{step}, batch size \(B\), model \(p\).
        \STATE Initialize Lagrange multipliers \(\boldsymbol{\lambda} \in \mathbb{R}_{\geq 0}^T\),
        initial policy, discriminator and baseline parameters \(\phi, \psi_{1:T}, \theta\)
        \WHILE{min over \(\boldsymbol{\lambda}\) not converged}
        \STATE Change \(\boldsymbol{\lambda}\) by coordinate descent
        \WHILE{inner max-min over \(\psi, \phi\) not converged}
        \FOR{i = 1, \ldots, \(n\)}
        \STATE Draw batch \(\tau_{1:B} \sim q_\phi(\tau)\)
        \FOR{\(t=1, \ldots, T\)}
        \STATE \(g_{\psi, t} = \nabla_\psi \tfrac{1}{B}\sum_{j=1}^B\log q_\psi(\tau_{u, j}|\tau_{a,j,t})\)
        \STATE \(\psi_t = \mathtt{step}(\psi_t, g_{\psi, t})\)
        \ENDFOR
        \ENDFOR
        \FOR{i = 1, \ldots, \(m\)}
        \STATE Draw batch \(\tau_{1:B} \sim q_\phi(\tau)\)
        \STATE \(w,r,g, g' = \mathtt{zeros}(T)\), \(g_1, g_2', g_\theta = 0\)
        \FOR{\(\tau \in \tau_{1:B}\)}
        \FOR{t = 1, \ldots, \(T\)}
        \STATE \(x_t = \tau_{x, t}, \ a_t = \tau_{a, t}, \ u_t =\tau_{u, t}\)
        \STATE \( v = \int_{a'_t}q(a'_t|x_t, u_t)p(x_{t+1}|a'_t, x_t, u_t)\)
        \STATE \(g_t = \nabla_\phi \left[\log q_\phi(a_t|x_t, u)\right]\)
        \STATE \(w_t = g_tp(x_{t+1}|a_t, x_{t}, u_t) / v\)
        \STATE \(r_t = \log q_\psi(u_t|a_t) - \log p(u_t)\)
        \STATE \(A_t = \left[\sum_{t'=t}^T r(x_{t'}, a_{t'}, u_{t'})\right] - \theta(x_t, u_t)\)
        \STATE \(g_2' = g_2' + g_tA_t\)
        \STATE \(g_\theta' = g_\theta' + \nabla_\theta\left[(\theta(x_t, u_t) - A_t)^2\right]\)
        \ENDFOR
        \STATE \(g'_{1} = g'_{1} + r_t\left(g_t + \sum_{t'=1}^{t-1}w_{t'}\right)\)
        \ENDFOR
        \STATE \(g_1 = \tfrac{1}{B}\sum_{t=1}^Tg'_1\), \(g_2 = \tfrac{1}{BT}\sum_{t=1}^Tg'_2\)
        \STATE \(\phi = \mathtt{step}(\phi, g_1 + g_2)\), \(\theta = \mathtt{step}(\theta, \frac{1}{BT}g'_\theta)\)        
        \ENDFOR
        \ENDWHILE
        \ENDWHILE
      \end{algorithmic}
    \end{algorithm}

    \begin{algorithm}[h]
      \label{alg:bedandsco}
      \caption{Model-Free MI-Constrained Policy Gradients}
      \begin{algorithmic}
        \STATE {\bfseries Input:} Constraint \(\epsilon \geq 0\),
        number of \(\psi\) gradient steps \(n\), number of \(\phi\) gradient steps \(m\),
        gradient-based optimization update \texttt{step}, batch size \(B\)
        \STATE Initialize Lagrange multipliers \(\lambda \in \mathbb{R}_{\geq 0}^T\),
        initial policy, discriminator and baseline parameters \(\phi, \psi, \theta\)
        \WHILE{min over \(\lambda\) not converged}
        \STATE Change \(\lambda\) by coordinate descent
        \WHILE{inner max-min over \(\psi, \phi\) not converged}
        \FOR{i = 1, \ldots, \(n\)}
        \STATE Draw batch \(\tau_{1:B} \sim q_\phi(\tau)\)
        \FOR{j = 1, \ldots, \(B\)}
        \STATE \(g_{\psi} = g_{\psi, t} + \nabla_\psi \log q_\psi(\tau_{u, j}|\tau_{a,j}, \tau_{x, j})\)
        \ENDFOR
        \STATE \(\psi = \mathtt{step}(\psi, \tfrac{1}{B}g_{\psi})\)
        \ENDFOR

        \FOR{i = 1, \ldots, \(m\)}
        \STATE Draw batch \(\tau_{1:B} \sim q_\phi(\tau)\)
        \STATE \(w,r,g, g' = \mathtt{zeros}(T)\), \(g_1, g_2', g_\theta' = 0\).  
        \FOR{\(\tau \in \tau_{1:B}\)}
        \FOR{t = 1, \ldots, \(T\)}
        \STATE \(\boldsymbol{x}_t = \tau_{x, t}, \ a_t = \tau_{a, t}, \ \boldsymbol{u}_t =\tau_{u, t}\)
        \STATE \(g_t = \nabla_\phi \left[\log q_\phi(a_t|x_t, u_t,)\right]\)
        \STATE \(A_t = \left[\sum_{t'=t}^T r(x_{t'}, a_{t'}, u_{t'})\right] - \theta(x_t, u_t)\)
        \STATE \(g_2' = g_2' + g_tA_t\)
        \STATE \(g_\theta' = g_\theta' + \nabla_\theta\left[(\theta(x_t, u_t) - A_t)^2\right]\)
        \ENDFOR
        \STATE \(g_1 = g_1 + \frac{\log q_\psi(\tau_u|\tau_{a}, \tau_{x})}{\log p(\tau_u)}\sum_{t=1}^Tg_t\)
        \ENDFOR
        \STATE \(g_1 = \tfrac{1}{B}\sum_{t=1}^Tg'_1\), \(g_2 = \tfrac{1}{BT}\sum_{t=1}^Tg'_2\)
        \STATE \(\phi = \mathtt{step}(\phi, g_1 + g_2)\), \(\theta = \mathtt{step}(\theta, \frac{1}{BT}g'_\theta)\)
        \ENDFOR
        \ENDWHILE
        \ENDWHILE
      \end{algorithmic}
    \end{algorithm}

    \begin{algorithm}[H]
      \label{alg:bedandsco}
      \caption{Reparameterised MI-Constrained PPO}
      \begin{algorithmic}
        \STATE {\bfseries Input:} Constraint \(\epsilon_t > 0\),
        gradient-based optimization update \texttt{step}, batch size \(B\), number of epochs \(n\), number of minibatches \(m\),
        surrogate horizon \(k\), number of truncated rollouts \(r\)
        \STATE Initialize Lagrange multipliers \(\lambda \in \mathbb{R}_{\geq 0}^T\),
        initial policy, discriminator \(\phi, \psi\), PID state \(\mathcal{P}\)
        \FOR{i = 1, \ldots, \(n\)}
        \STATE Draw batch \(\tau_{1:mB} \sim q_\phi(\tau)\)
        \FOR{j = 1, \ldots, \(m\)}
        \STATE \(\ell_\text{PPO} = \operatorname{PPO}(\tau_{jB:(1+j)B}, \phi)\)
        \STATE Subsample \(r\) states, actions  \((x, u, a)\) from \(\tau_{jB:(1+j)B}\).
        \STATE \(\mathcal{L}_\psi = -\sum_{l=1}^{B}\sum_{c=1}^r\log q_\psi(u_{l,c} | a_{l,\max\{0, c-k\}:c})\)
        \STATE \(I= 0\)
        \FOR{l = 1, \ldots, \(r\)}
        \STATE Roll out \(\tau'\), the trajectory with initial state \((x_l, u_l)\) and lasting for \(k\) steps with policy \(\phi\)
        \STATE \(I = I + \sum_{c=1}^r \log (q_\psi(u'_c | a'_{1:c})) - \log q(u'_c)\)
        \ENDFOR\\
        \STATE \(I = \operatorname{stop-grad}(I, \psi)\)
        \STATE \(\mathcal{L} =  \mathcal{L}_\text{PPO} + \mathcal{L}_\psi + I\lambda\)
        \STATE \(\mathcal{P}, \lambda = \operatorname{PID-Update}(\mathcal{P}, I, \epsilon_t)\)
        \STATE \(\psi, \phi = \mathtt{step}((\psi, \phi), \nabla_{\psi, \phi}\mathcal{L} )\)
        \ENDFOR
        \ENDFOR
      \end{algorithmic}
    \end{algorithm}

    \begin{algorithm}[H]
      \label{alg:pid_update}
      \caption{PID Controller Update}
      \begin{algorithmic}
        \STATE {\bfseries Input:} MI constraint \(\epsilon\), PID parameters \(k_p, k_i, k_d\), integrated error \(\sigma\), MI at previous iterate \(I_\text{prev}\), MI \(I\).
        \STATE \(\delta = I - \epsilon\)
        \STATE \(\partial = \max\left\{0, I - I_\text{prev}\right\}\)
        \STATE \(\sigma = \max\left\{0, \sigma + \delta\right\}\)
        \STATE \(\lambda = \max\left\{0, k_p \delta + k_i \sigma + k_d \partial\right\}\)
        \STATE \(I_\text{prev} = I\)
        \STATE {\bfseries Output:} \(\lambda, I_\text{prev}, \sigma\)
      \end{algorithmic}
    \end{algorithm}

\newpage
\newpage
\vspace{200cm}
\subsection{Additional Experimental Details}
\label{sec:addit-exper-deta}
\paragraph{Score-Based and Model-Free Estimator}
Experiments on the score-based and model-free estimator were run on a dual-core 3.5GHz Intel i7 CPU, with the runs taking around one minute to complete.
We selected the hyperparameters by hand, observing which ones resulting in convergence.
In the case of the Lagrange multipliers \(\lambda\), we tried a few different settings to see which was best for equation \eqref{eq:main-dual-problem}.
We used the Adam optimizer \cite{kingmaAdamMethodStochastic2015} for gradient-based optimization.
\\
\paragraph{Reparameterised Estimator}
For experiments with the reparameterised estimator, we used an A4000 GPU and CUDA-based JAX~\citep{bradburyJAXComposableTransformations2020} and BRAX~\citep{freeman2021brax}. The predictor was a causally-masked transformer with learned position embeddings. The embeddings were learned for positions 1 to \(T\). The predictor was trained using sequences of length \(k\), from time indices \(t\) to \(t +k\). The BRAX PPO implementation was modified to add the predictor parameters to the training state. At each PPO minibatch step, \texttt{n-truncated-rollouts} states were sampled from the states fed into the loss. Then, these states were used as initial states to roll out \(k\) further steps using the current policy. The approximate mutual information \(I\) was calculated using the current predictor parameters and \(I \lambda\) was added to the loss, where \(\lambda\) was the current Lagrange multiplier. Additionally, the cross-entropy loss from the predictor between the prediction of \(u\) and the ground-truth \(u\) was added to the loss, training the predictor by maximum-likelihood. We used the same AdamW~\citep{loshchilov2018decoupled} optimizer for both the policy and the predictor, and used gradient clipping to stabilise training. We found that the same learning rate could be used for the policy and the predictor. For the plots in figure \ref{fig:pusher-ant-mi-reward}, we used a horizon of 10 and \(\epsilon\) of 0.001, 0.01, 0.1, 0.3 respectively. The networks for the policy and value function were left as the BRAX defaults: the policy as an MLP with four layers of 32 hidden nodes, and the value network as an MLP with five layers of 256 hidden nodes. Both networks had swish nonlinearities~\citep{ramachandran2017searching}. Due to a configuration error, the policy network for the  ant-constrained case presented in the main paper was given by three layers of 64 hidden nodes, and four hidden layers of 256 nodes. However, the results were very similar over different network sizes. 

For the determination of the Lagrange multiplier \(\lambda\), we used a PID controller, following~\citet{stooke2020responsive}. In particular, we use the standard PID update procedure, reproduced in algorithm \ref{alg:pid_update}. 

\begin{table}[]
\begin{tabular}{@{}lllll@{}}
\toprule
Hyperparameter     & Value              &  &  &  \\ \midrule
\(r^*\)            & 1.                 &  &  &  \\
\(r^-\)            & 0.5                &  &  &  \\
\(r^{\text{VPN}}\) & 0.9                &  &  &  \\
T                  & 10                 &  &  &  \\
Batch Size         & 32                 &  &  &  \\
Number of Epochs   & 5000               &  &  &  \\
Learning Rate      & 3\(\cdot 10^{-3}\) &  &  &  \\ \bottomrule \\
\end{tabular}
\caption{Internet Connectivity Hyperparameters}
\end{table}

\begin{table}[]
\begin{tabular}{@{}lllll@{}}
\toprule
Hyperparameter                     & Value              &  &  &  \\ \midrule
Environment Force Noise \(\sigma\) & 0.5                &  &  &  \\
Initial Position \(\sigma\)        & 1.                 &  &  &  \\
T                                  & 10                 &  &  &  \\
Batch Size                         & 128                &  &  &  \\
Number of Epochs                   & 4000               &  &  &  \\
Learning Rate                      & 3\(\cdot 10^{-3}\) &  &  &  \\ \bottomrule \\
\end{tabular}
\caption{2d Control Hyperparameters}
\end{table}

\begin{table}[]
\begin{tabular}{@{}lllll@{}}
\toprule
Hyperparameter                  & Value              &  &  &  \\ \midrule
Episode Length              & 100                    &  &  &  \\
  Learning Rate       & \(3\times 10^{-4}\) &  &  &  \\
  Entropy Cost & \(1\times 10^{-2}\) &  &  &  \\
  Batch Size                      & 512 &  &  &  \\
  Environment Steps & \(5 \times 10^9 \) &  &  &  \\
  Normalize Observations& True &  &  &  \\
  Num Minibatches & 16 &  &  &  \\
  Discount Rate & 0.95 &  &  &  \\
  Reward Scaling & 5 &  &  &  \\ 
  Number of Truncated Rollouts & 128 &  &  &  \\
  PID \(k_p, k_i, k_d\)& 1, 0.1, 0.001 &  &  &  \\
  Gradient Clip Level & 10 &  &  &  \\
  Weight Decay & \(1\times 10^{-6}\) &  &  &  \\
  \(\psi\) Layers & 6 &  &  &  \\
  \(\psi\) Embedding Size & 256 &  &  &  \\
  \(\psi\) Num Heads & 4 &  &  &  \\
  \(\psi\) Dropout & 0.1 &  &  &  \\
  Misc.  & BRAX PPO Example Default &  &  &  \\
  \bottomrule \\  
\end{tabular}
\caption{Reparameterized (PPO) Hyperparameters}
\end{table}

\begin{table}[]
\begin{tabular}{@{}lllll@{}}
\toprule
Hyperparameter       & Value              &  &  &  \\ \midrule
Income KDE Bandwidth & \$10,000           &  &  &  \\
T                    & 4                  &  &  &  \\
Batch Size           & 128                &  &  &  \\
Number of Epochs     & 1000               &  &  &  \\
Learning Rate        & 1\(\cdot 10^{-3}\) &  &  &  \\ \bottomrule \\
\end{tabular}
\caption{SNAP Hyperparameters}
\end{table}

\begin{table}[]
\begin{tabular}{@{}lllll@{}}
\toprule
Hyperparameter                  & Value              &  &  &  \\ \midrule
Initial Separation              & 2                  &  &  &  \\
Dynamics Noise \(\sigma\)       & 0.25               &  &  &  \\
Initial Distribution \(\sigma\) & 0.5                &  &  &  \\
T                               & 6                  &  &  &  \\
Batch Size                      & 12                 &  &  &  \\
Number of Epochs                & 5000               &  &  &  \\
Learning Rate                   & 1\(\cdot 10^{-4}\) &  &  &  \\ \bottomrule \\
\end{tabular}
\caption{Customer Service Hyperparameters}
\end{table}

\newpage
\subsection{Control of Variance with Truncated Predictor}
\label{sec:control-variance}
In this section we show that in the reparameterization case, the gradient norm increases with increased predictor horizon \(k\).
Figure \ref{fig:grad-norm-horizon} shows that as the predictor is trained on increasingly truncated sequences (i.e. \(k\) decreases),
the gradient norm generally decreases. Training on the full horizon, with \(k=100\), is generally unfeasible for this reason. We found that above \(k=20\), the training was unstable and would frequently diverge.

\begin{figure}\centering
  \includegraphics[width=\textwidth]{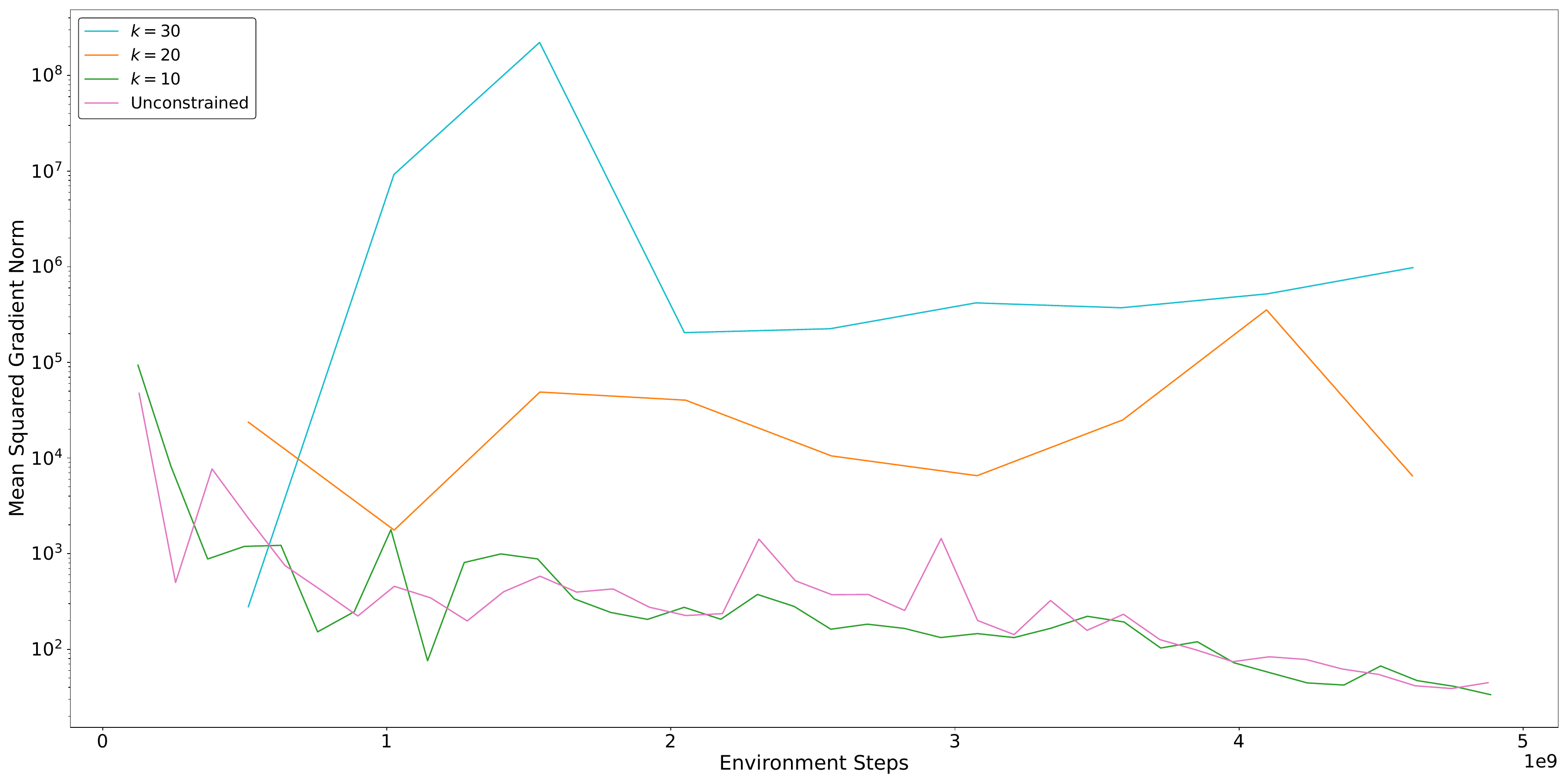}
  \caption{Gradient norm for different horizon lengths \(k\) during training. Note the logarithmic scale. As \(k\) increases, the gradient norm increases
  dramatically.}
  \label{fig:grad-norm-horizon}
\end{figure}

\newpage
\subsection{Upper Bound Loosesness with Shielded \(u\).}
\label{sec:u-shielded}
If we are completely unable to form a dynamics estimator we may use \(I_q(\tau_a, \tau_x;\tau_u)\)
to constrain \(I_q(a_t;u_t)\). However, there is a significant class of problems for which the
upper bound is loose, which we call \(u\)-shielded MDPs.

\begin{figure}[htb]
  \label{graph-model}
  \caption{}
\begin{center}
  \begin{tikzpicture}
 \node [align=center, thick, text          width=0.7cm, draw,        circle, radius=0.5, name=u]    at            (-2,0)   {$\boldsymbol u$};
 \node [align=center, thick, text          width=0.7cm, draw,        circle, radius=0.5, name=x1]    at            (-0.5,0)   {$\boldsymbol x_1$};
 \node [align=center, thick, text         width=0.7cm, draw,   circle,     radius=0.5, name=a1] at            (1,         0)           {$\boldsymbol a_1$};
 \node [align=center, thick, text         width=0.7cm, draw,   circle,     radius=0.5, name=x2] at            (2.5,         0)           {$\boldsymbol x_2$};
 \node [align=center, thick, text         width=0.7cm, draw,   circle,     radius=0.5, name=a2] at            (4,         0)           {$\boldsymbol a_2$};
 \node [align=center, draw=none, name=x3] at            (5.5,         0)           {}; 
 \draw [thick,        ->]          (u)         --      (x1);
 \draw [thick,        ->]          (x1)         --      (a1);
 \draw [thick,        ->]          (a1)         --      (x2);
 \draw [thick,        ->]          (x2)         --      (a2);
 \draw [thick,        ->]          (x1)         to[out=40, in=140]      (x2);
 \draw [dotted,        ->]          (x2)         to[out=40, in=140]      (x3);
 \draw [dotted,        ->]          (a2)         --      (x3); 
\end{tikzpicture}
\end{center}
\caption{A \(u\)-shielded MDP}
\end{figure}

They have the property where the protected variable \(u\) influences the initial distribution,
but subsequently \(u\) doesn't directly influence the dynamics or the rewards of the MDP. An idealised
example would be a banking system where minority status influences initial income distribution,
but (conditioned on income) the default rate doesn't depend on minority status, and the banking system's performance is evaluated
without reference to minority status. In that case the policy \(q(a|x,u)\) doesn't need to depend on
\(u\) at all, in the sense that for any policy \(q(a|x,u)\) there is another policy \(q(a|x,u) = q(a|x)p(u)\)
with the same reward. Hence without loss of generality we assume that all policies in the \(u\)-shielded setting are in
this \(u\)-independent form.

In that case, we have that
\begin{lemma}
    In a \(u\)-shielded MDP, and a greedy policy \(q_{\text{greedy}}\) which is reward-maximizing
  under no mutual information constraint, \(I_q(\tau_x, \tau_a ; u ) \geq I_{q_{\text{greedy}}}(a_t;u)\) for any
    policy \(q\).
\end{lemma}
\begin{proof}
  We have \(I_q(\tau_x, \tau_a; u) = \sum_{t=1}^T I_q(x_t, a_t;u|\tau_{a_{1:t}}, \tau_{x_{1:t}}) =
  I_q(x_1;u)\), by the decomposition of mutual information and the fact that \(x_t\) and \(a_t\) are
  conditionally independent from \(u\) given \(x_1\). The data processing inequality on the
  Markov chain \(u \rightarrow x_1 \rightarrow a_t\) gives us \(I_\text{greedy}(a_t;u) \leq
  I_p(x_1;u)\). Therefore, we have \(I_\text{greedy}(a_t;u) \leq I_q(\tau_x, \tau_a;u)\).
\end{proof}
In particular, this result tells us that if we constrain our upper bound in the hope of reducing \(I(a_t;u)\),
if our problems have the \(u\)-shielded property then our constraint will never exclude the greedy policy.

\newpage
\subsection{Equivalence to Demographic Parity in Single-Timestep Case}
\label{sec:equiv-demogr-parity}
By writing the binary classification problem with a demographic parity constraint
as an RL problem, we can show its equivalence to our privacy constraint in the single-timestep case.
Since classification has no concept of feedback, we can describe the demographic parity constrained classification problem as an episodic MDP with only one timestep. We define the \emph{MDP-classification setting} as the MDP with \(T=1\), action space
\(\mathcal{A} = \{0, 1\}\) corresponding to the two possible labels, state space \(\mathcal{X} \times \mathcal{U}\), and a reward distribution
\(r(x_1, u_1, a_1) = -\mathds{1}\left[a_1 = y\right]\) for random variable \(y \sim p(y|x_1, u_1)\) (for \(\mathcal{Y} =
\{0, 1\}\)) corresponding to 0-1 loss.
\begin{lemma}
  Problem \eqref{eq:main-problem} with the MDP-classification setting
  is equivalent to the fair classification problem with the 1-0 loss 
  and the generalised demographic parity (DP) fairness constraint (\(I(\hat y; u) < \epsilon\)) for all
  \(p(x, u, y)\). Furthermore, the mutual information constraint in equation \eqref{eq:main-problem} is the
  only choice of constraint with this equivalence.
\end{lemma}
\begin{proof}
  To solve problem \eqref{eq:main-problem} with the specified dynamics requires finding a policy
  \(q(a_1|x_1, u_1)\) that maximizes \(\mathbb{E}_{a_1, x_1, u_1, y \sim q(a_1, x_1, u_1, y)}\left[\mathds{1}\left[a_1
      = y\right]\right]\), subject to \(I(u_1; a_1) < \epsilon\). Solving the fair classification problem with demographic parity requires learning a model that emits a classification \(\hat Y\)
  maximizing \(\mathbb{E}_{\bx, \bu, Y \sim p(\bx, \bu, Y)}\left[\mathds{1}\left[\hat Y =
      Y\right]\right]\) subject to \(I(\hat Y;\bu) < \epsilon\). Hence we have the same objective, and so the
  problems are the same if and only if they have the same constraints, i.e. we have the constraint \(I(a_1;u_1)\).
\end{proof}

\end{document}